%% file: main.tex
\documentclass[10pt,twocolumn,letterpaper]{article}

\usepackage{cvpr}              
\input{preamble}
\definecolor{cvprblue}{rgb}{0.21,0.49,0.74}
\usepackage[pagebackref,breaklinks,colorlinks,allcolors=cvprblue]{hyperref}

\def\paperID{} 
\def\confName{CVPR}
\def\confYear{2026}

\title{ParallelVLM: Lossless Video-LLM Acceleration with Visual Alignment Aware Parallel Speculative Decoding}

\author{Quan Kong$^{*}$, Yuhao Shen$^{*}$, Yicheng Ji, Huan Li, Cong Wang$^{\dag}$ \\
Zhejiang University, Hangzhou, China\\
{\tt\small \{qkv, riven, jiyicheng.cs, lihuan.cs, cwang85\}@zju.edu.cn}
}

\usepackage{float}
\usepackage{graphicx,overpic}
\usepackage{algorithm}            
\usepackage{algpseudocode}        
\usepackage{amsmath}
\usepackage{colortbl}
\usepackage{color,xcolor}
\definecolor{light-gray}{gray}{0.95}    
\usepackage{multirow}
\usepackage{bm}
\usepackage{enumerate}
\usepackage{pifont}
\usepackage{amsthm} 
\usepackage{makecell}
\newtheorem{theorem}{Theorem} 

\theoremstyle{definition}

\usepackage{subcaption}
\begin{document}
\twocolumn[{%
\maketitle
\includegraphics[width=0.76\linewidth]{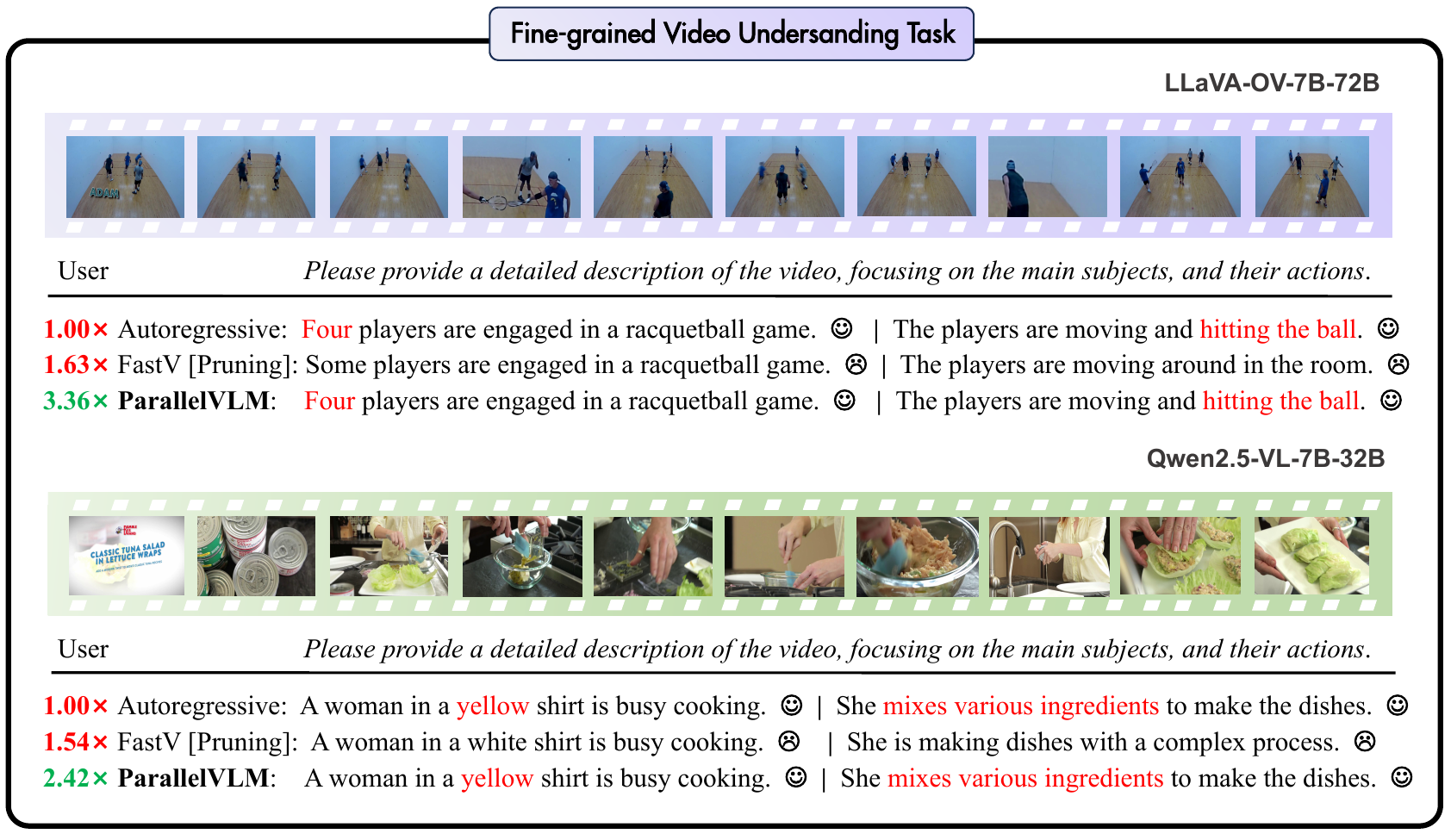}
\includegraphics[width=.23\linewidth]{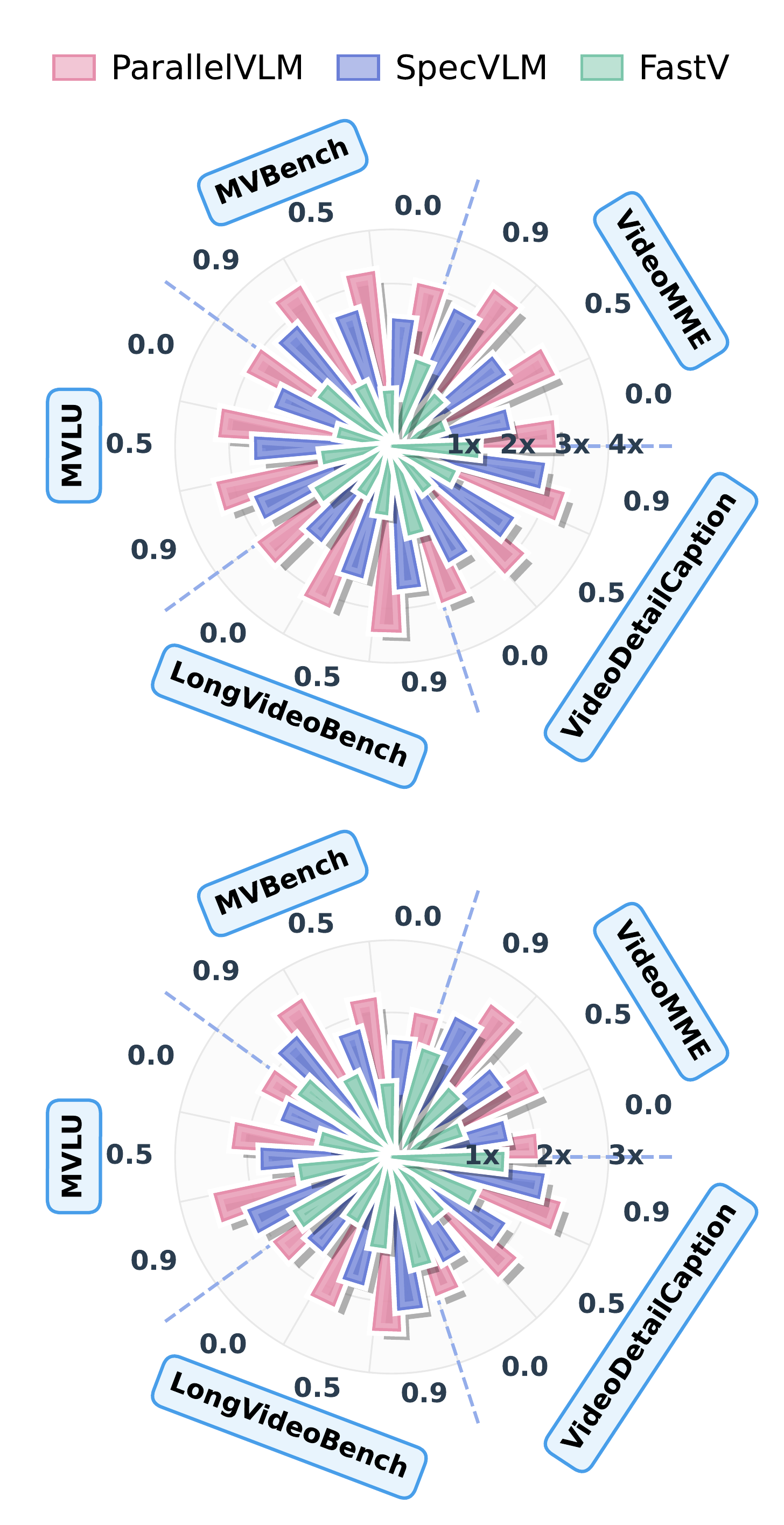}
\vspace{-0.5em}
\captionof{figure}{Comparison of ParallelVLM with vanilla autoregressive decoding and visual token pruning. ParallelVLM achieves up to $3.36\times$ acceleration for LLaVA-OV-72B and $2.42\times$ for Qwen2.5-VL-32B across various video understanding benchmarks. \vspace{1em}}
\label{fig:teaser}
}]

\renewcommand{\thefootnote}{\fnsymbol{footnote}}
\setcounter{footnote}{0}
\footnotetext{$^*$Equal Contribution.}
\footnotetext{$^{\dag}$Corresponding Author.}

\input{sec/0_abstract}    
\input{sec/1_intro}
\input{sec/2_preliminary}
\input{sec/3_motivation}

\input{sec/4_method}

\input{sec/5_experiment}
\input{sec/6_related_work}

\input{sec/7_conclusion}
{
    \small
    \bibliographystyle{ieeenat_fullname}
    \bibliography{main}
}

\input{sec/X_suppl}

\end{document}

%% file: sec/0_abstract.tex
\begin{abstract}
Although current Video-LLMs achieve impressive performance in video understanding tasks, their autoregressive decoding efficiency remains constrained by the massive number of video tokens. Visual token pruning can partially ease this bottleneck, yet existing approaches still suffer from information loss and yield only modest acceleration in decoding. In this paper, we propose ParallelVLM, a training-free draft-then-verify speculative decoding framework that overcomes both mutual waiting and limited speedup-ratio problems between draft and target models in long-video settings. ParallelVLM features two parallelized stages that maximize hardware utilization and incorporate an Unbiased Verifier-Guided Pruning strategy to better align the draft and target models by eliminating the positional bias in attention‑guided pruning. Extensive experiments demonstrate that ParallelVLM effectively expands the draft window by $1.6\sim1.8\times$ with high accepted lengths, and accelerates various video understanding benchmarks by 3.36$\times$ on LLaVA-OneVision-72B and 2.42$\times$ on Qwen2.5-VL-32B compared with vanilla autoregressive decoding. Codes are available at \href{https://github.com/imKQv/ParallelVLM}{https://github.com/imKQv/ParallelVLM}.
\end{abstract}

%% file: sec/1_intro.tex
\begin{figure*}
    \centering
\includegraphics[width=1.0\linewidth]{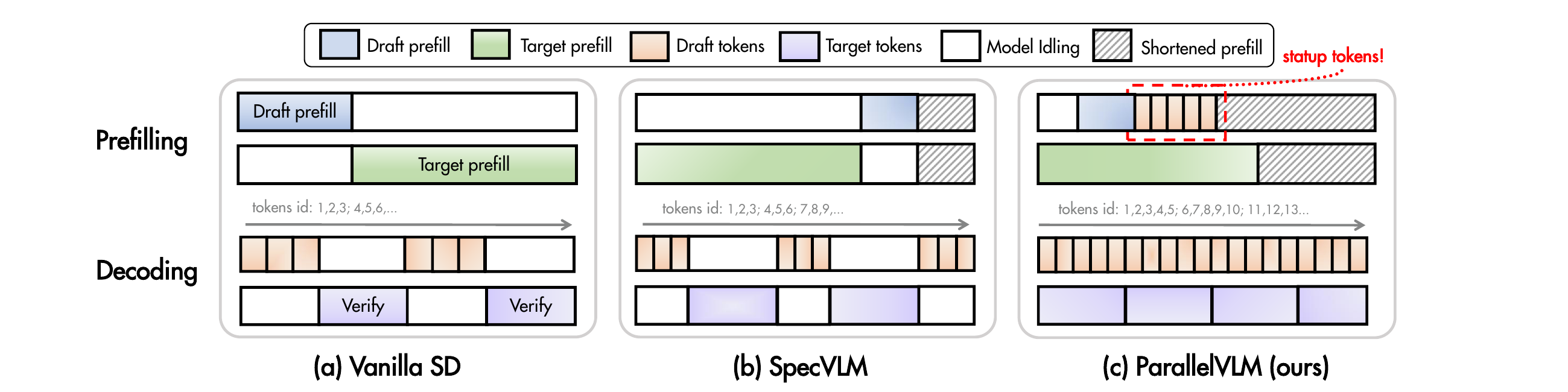}
    \vspace{-0.3in}
    \caption{\small{Comparison of different training-free speculative decoding frameworks for Video-LLMs. (a) Vanilla SD: sequential in both prefilling and decoding stages, achieving speedup only with the draft-then-verify design. (b) SpecVLM~\cite{Ji2025SpecVLMES}: inherits vanilla SD and enables faster draft model execution via attention-guided pruning. (c) ParallelVLM: adopts a parallel pipeline to mitigate idling time in vanilla SD, co-designed with UV-Prune strategy to expand the draft window size meanwhile preserving draft/target alignment.}}
    \label{fig:comparison}
\vspace{-0.2in}
\end{figure*}

\vspace{-0.2in}
\section{Introduction}
\label{sec:intro}

The advent of Large Language Models (LLMs) has ushered a new era in AI with foundational models like GPT, LLaMA, Qwen and Deepseek~\cite{Achiam2023GPT4TR,DeepSeekAI2024DeepSeekV3TR,yang2025qwen3,Touvron2023LLaMAOA}. This success has quickly catalyzed multimodal research~\cite{Liu2023VisualIT,hurst2024gpt,Li2024MiniGeminiMT} such as Video-LLMs~\cite{Li2024LLaVAOneVisionEV,Lin2023VideoLLaVALU,Bai2025Qwen25VLTR}, enabling a wide range of new applications like video understanding~\cite{Fang2024MMBenchVideoAL}, content auditing~\cite{Liu2025ProtectingYV}, and autonomous driving~\cite{shao2024lmdrive}.

Despite these advances, the practical deployment of Video‑LLMs is hampered by \emph{inference latency}~\cite{zhang2026efficientinferencelargevisionlanguage}. State-of-the-art models like LLaVA-OneVision~\cite{Li2024LLaVAOneVisionEV} typically encode videos into long sequences of visual tokens, which can easily scale from thousands to millions. Due to the quadratic computational cost of self-attention~\cite{Vaswani2017AttentionIA}, such an explosion in token length creates a crippling performance bottleneck in both prefilling (KV cache computation) and decoding. To address this challenge, numerous efforts have been devoted to designing more effective visual token pruning methods~\cite{FastVChen2024AnII, Zhang2024SparseVLMVT, Xing2024PyramidDropAY,Tao2024DyCokeD}. However, as shown in Fig.~\ref{fig:teaser}, since autoregressive access to the full model parameters is still required, they only offer limited acceleration. Even worse, \emph{pruning without validation} induces inevitable distributional shifts that cause measurable performance degradation, particularly in fine-grained video understanding tasks.

To achieve lossless acceleration, \emph{Speculative Decoding} (SD) offers a promising alternative~\cite{Chen2023AcceleratingLL,Leviathan2022FastIF,Xia2022SpeculativeDE} and SpecVLM has pioneered the application of SD to Video-LLMs~\cite{Ji2025SpecVLMES}. As shown in Fig.~\ref{fig:comparison} (a), SD replaces autoregressive decoding with a draft-then-verify paradigm, achieving a theoretical speedup proportional to the speed ratio of the draft/target model pairings. However, when applied to Video-LLMs, existing SD frameworks face two new challenges. \textbf{(i) Sequential Execution}: SD enforces strict sequential execution indicated by the idling gaps in Fig.~\ref{fig:comparison} (a/b), which becomes costly with massive visual token prefixes in Video-LLMs. \textbf{(ii) Entanglement between Speed Ratio and Model Alignment}: visual tokens often contain substantial redundancy and the draft model can be pruned to improve the draft/target speed ratio. However, draft models with pruned visual tokens often struggle to retain both salient visual details and coherent textual grounding to align with the target model's distribution. While SpecVLM attempts to guide the draft model pruning with the target's attention scores~\cite{Ji2025SpecVLMES}, it suffers from the ``positional bias'' problem~\cite{Endo2024FeatherTT}, which in turn hurts model alignment for videos rich in temporal and spatial semantic cues.   

To address these challenges, we propose ParallelVLM, a co-design of parallel SD and draft video token pruning. Our framework combines prefilling/decoding parallelism and an Unbiased Verifier-Guided Pruning (UV-Prune) strategy to expand the draft window size without sacrificing alignment. As illustrated in Fig.~\ref{fig:comparison} (c), ParallelVLM entails Parallel Prefilling (PP) and Parallel Decoding (PD) stages: 1) in PP, the draft model’s \emph{pruned} prefilling executes concurrently with the target’s \emph{full} prefilling, in which pruning decisions are guided by the vision-text variation signals from the target’s early-layer representations. This mitigates positional bias and hides the draft’s pruning/start‑up timing cost under the target’s latency. 2) In PD, draft token generation and target verification proceed in overlapping windows with an optimal window size determined by the enlarged draft/target speed ratio. The draft/target pairings also preserve a high accepted length with the alignment-aware design. Our contributions are summarized as follows:
\begin{itemize}
    \item[\ding{71}] \textbf{Systematic Analysis.} We conduct a comprehensive analysis on SD efficiency for Video-LLMs, identifying the sequential execution bottleneck and the coupling relations between draft/target speed ratio and model alignment.
    \item[\ding{71}] \textbf{Parallel Acceleration Framework.} We propose ParallelVLM, a \emph{training-free and lossless} architecture that leverages parallelism and unbiased pruning to expand the draft window while maintaining robust acceptance rates.
    \item[\ding{71}] \textbf{Significant Speedups.} Extensive experiments on five video understanding benchmarks show ParallelVLM achieves up to \textbf{3.36$\times$} and \textbf{2.42$\times$} speedups for the LLaVA-OneVision-72B and Qwen2.5-VL-32B respectively.
\end{itemize}

%% file: sec/2_preliminary.tex
\section{Preliminary}
\label{sec:preliminary}
\subsection{Video Large Language Models (Video-LLMs)}

A typical Video-LLM consists of: 1) a vision encoder $E$ that transforms input video frames into a sequence of feature vectors; 2) an adapter $g_\theta$ that maps these visual features into the language model's embedding space; and (3) a large language model that processes the combined multimodal input~\cite{Li2024LLaVAOneVisionEV, Bai2025Qwen25VLTR}. Formally, a video input is processed into a sequence of visual tokens $V_{1:m}:=[V_1,...,V_m]$ and the text is tokenized into $X_{1:n}:=[X_1,...,X_n]$. These sequences are concatenated to form a long, multimodal context prefix $[V_{1:m},X_{1:n}]$, where $m$ could be much larger than $n$ for Video-LLMs. Denote $p(\cdot)$ as the model distribution, then the subsequent text is generated autoregressively by conditioning on the text prefix of length $k\ (k\ge n)$ :
\begin{eqnarray}
X_{k+1}\sim p(\ \cdot\ |V_{1:m},X_{1:k}).
\label{video-llm}
\end{eqnarray}
Current visual token pruning methods curtail the video tokens $V_{1:m}$ before or inner LLM processing, changing the generated token distribution to $X_{k+1}^{'}\sim p^{'}(\ \cdot\ |V_{1:m}^{'},X_{1:k})$, inevitably leading to performance degradation.

\begin{figure}[t]
    \centering
    \includegraphics[width=\linewidth]{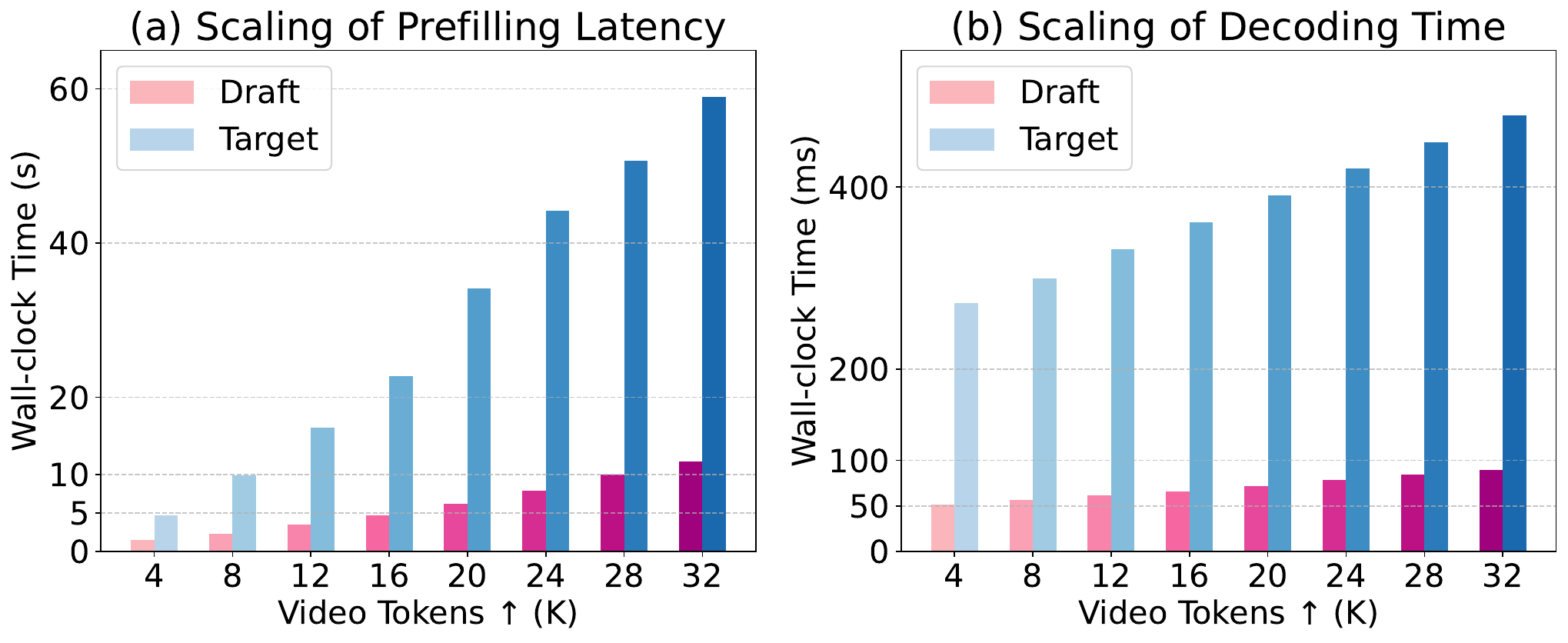}
    \vspace{-0.3in}
    \caption{\small{Scaling of (a) Prefilling Latency and (b) Decoding Time with increasing number of video tokens. Experiments are conducted on LLaVA-OV 7B draft model and 72B target model.}}
    \label{fig:motivation}
\end{figure}

\subsection{Speculative Decoding (SD)}
\label{subsec:pre_sd}
SD accelerates autoregressive LLM inference without altering their output distribution~\cite{Chen2023AcceleratingLL, Leviathan2022FastIF, Xia2022SpeculativeDE}. The core of SD operates on a \emph{draft-then-verify principle}: a lightweight \emph{draft model} $M_q$ generates $\gamma$ candidate tokens. These tokens are then verified in a single, parallel forward pass by the more powerful target model $M_p$. Denote the single forward time of the draft and target model as $T_q$ and $T_p$, respectively, then the speed ratio $c=T_p/T_q$ normally satisfies $c\gg1$. 

The lossless nature of SD is guaranteed by its verification mechanism. With accepted length $k$, for each proposed token $\hat{X}_{k+t}$ in the draft sequence $\hat{X}_{k+1:k+\gamma}$, it is accepted with a probability determined by rejection sampling, 
\begin{eqnarray}
P(\text{accept}) = \mathrm{min} \left\{1,\frac{p(\hat{X}_{k+t}|X_{1:k},\hat{X}_{k+1:k+t-1})}{q(\hat{X}_{k+t}|X_{1:k},\hat{X}_{k+1:k+t-1})}\right\}\nonumber,
\label{target_sampling}
\end{eqnarray}
where $p(\cdot)$ and $q(\cdot)$ are the probability distributions of the target and draft models. If $\hat{X}_{k+t}$ is rejected, then a new token is sampled from $\mathrm{norm}(\mathrm{max}(0,p(X_{k+t})-q(X_{k+t})))$. 

%% file: sec/3_motivation.tex
\section{Why SD Falls Short for Video-LLMs?} 
\label{sec:motivation}
\noindent\textbf{Challenge 1: Sequential Execution Bottleneck}. In vanilla SD, the draft and target models execute prefilling and decoding sequentially on the entire video token sequence. As shown in Fig.~\ref{fig:motivation}, the prefilling latency and decoding time become heavier with the increase of video tokens for both the draft and target models. For example, with $24$K video tokens, the target’s prefilling takes $44.23$ s, and the draft prefilling adds another $7.92$ s. When it turns to decoding, the draft decoding time is $T_q=78$ ms and the target decoding accounts for  $T_p=420$ ms. With a typical draft window size of $\gamma=5$ for the LLaVA-OV-7B/72B combination, the entire cycle takes $\gamma \cdot T_q  + T_p \approx 2 T_p$. This means that the hardware remains idle for almost 20\% of the span during prefilling and 50\% during decoding due to sequential scheduling. As Video-LLMs require more powerful draft models to understand videos, the sequential waiting becomes even more expensive and persists as a \emph{fixed tax}, whereas the vanilla SD provides no measure to ``hide'' this overhead. 

\begin{figure}[t]
    \centering
\includegraphics[width=1.0\linewidth]{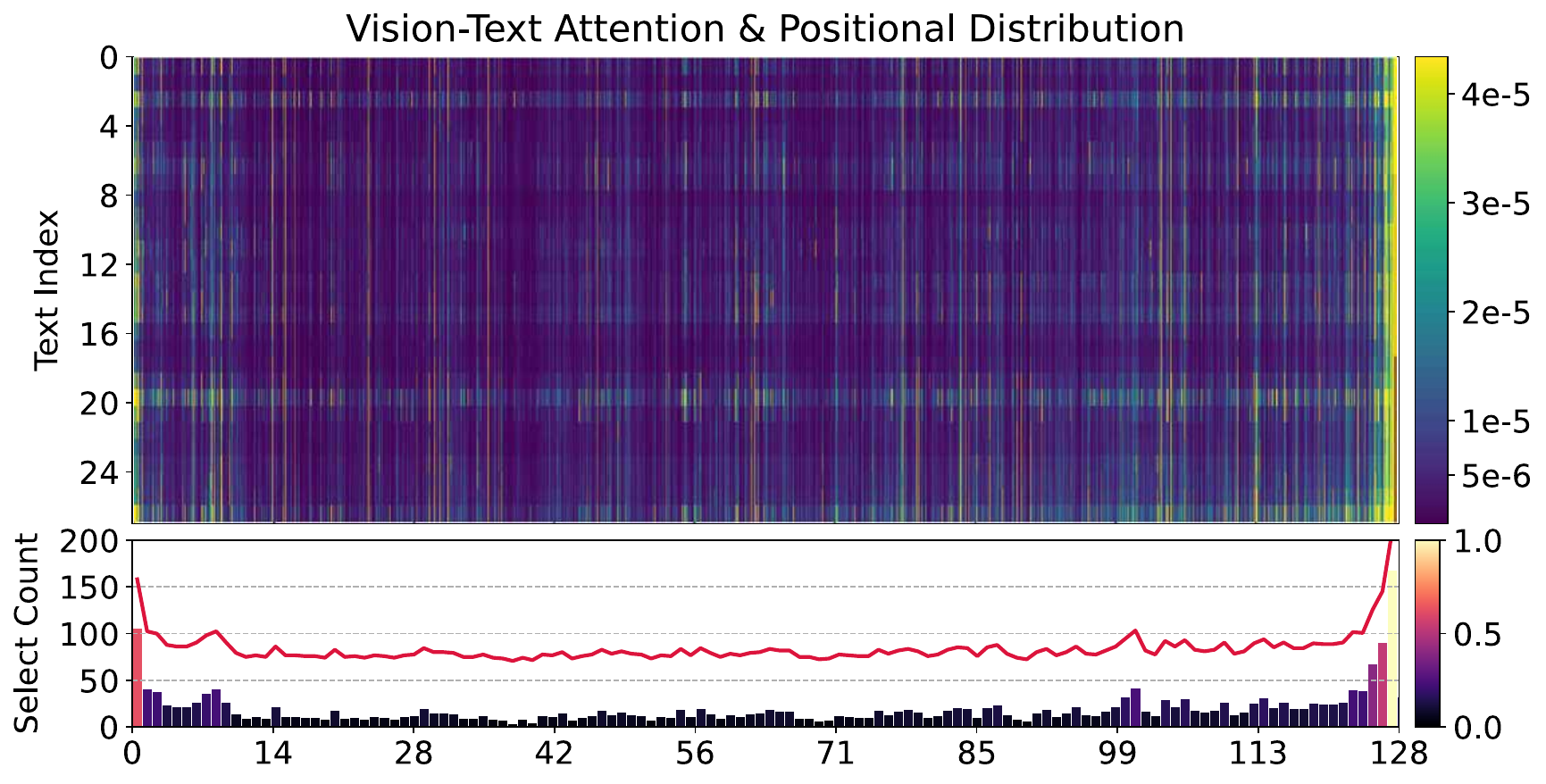}
    \vspace{-0.3in}
    \caption{\small{Target model's attention guidance for draft video token pruning. We observe positional bias on both ends of the video tokens. When retention rate is 10\%, frames 1, 125-128 accumulate up to 20.9\% of total selected tokens within only 4\% position width.}}
    \label{fig:attn}
\vspace{-0.2in}
\end{figure}

\begin{figure*}
    \centering
\includegraphics[width=1.0\linewidth]{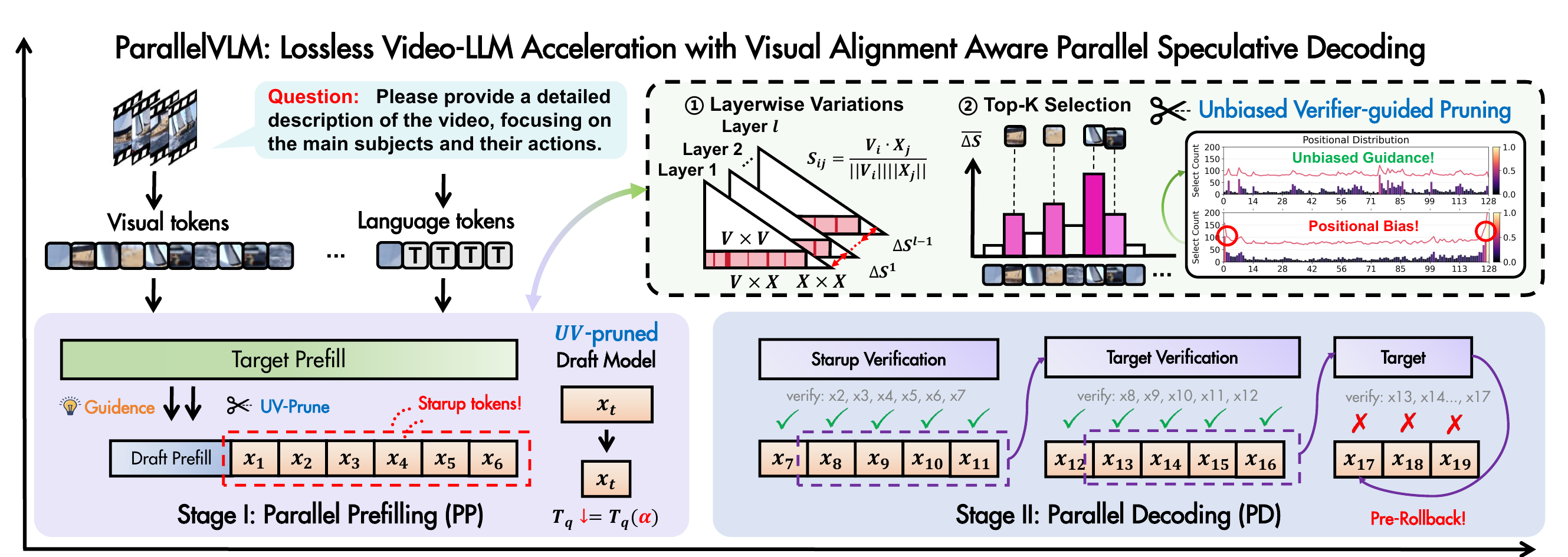}
    \caption{\small{The proposed ParallelVLM consists of Parallel Prefilling (PP), Parallel Decoding (PD) and the Unbiased Verifier-guided Pruning (UV-Prune). (a) During PP stage: draft model prefilling and pruning executes \emph{in parallel} with target model prefilling, where UV-Prune transfers salient alignment semantics from the target model to the draft model without ``positional bias". (b) During PD stage: the reduced draft decoding time $T_q$ enables an enlarged window size $\gamma$ \emph{in parallel} with verification, meanwhile maintaining high acceptance rates $\tau$.}}
    \label{fig:architecture}
\vspace{-0.2in}
\end{figure*}

\noindent\textbf{Challenge 2: Entanglement between Speed Ratio $c$ and Model Alignment}. The achievable acceleration in SD is governed by: (i) \textbf{speedup ratio $c=T_p/T_q$}, which decides how many tokens can be speculated within the target’s verification time; (ii) \textbf{model alignment} between draft and target outputs, which controls the acceptance rates. Since video token sequences are known to have high redundancy~\cite{FastVChen2024AnII,Bolya2022TokenMY}, $T_q$ can be reduced by pruning the draft visual tokens to improve $c$. However, a smaller draft model with pruned tokens often struggles to retain both salient visual details and coherent textual grounding to align with the target model's full-context distribution. This consequently results in a sharp decrease of the acceptance rates. 

To maintain alignment, effective signals from the target model should be utilized to guide the pruning process. SpecVLM~\cite{Ji2025SpecVLMES} leverages attention scores as guidance~\cite{Ji2025SpecVLMES}. Although seemingly intuitive, this approach is undermined by the \textbf{positional bias} phenomenon. As shown in Fig.~\ref{fig:attn}, 21\% of the attention-guided video tokens are chosen within only 4.0\% position width (frames 1, 125-128) simply because they are ``attention sinks''~\cite{Xiao2023EfficientSL} or close to the text query~\cite{Endo2024FeatherTT}, not because they are semantically important for visual reasoning and textual grounding. This is particularly problematic for video tasks where temporal coherence and mid-video content matter critically. In addition, attention-guided pruning is incompatible with FlashAttention~\cite{Dao2023FlashAttention2FA} that adopts block-wise kernels without explicit computation of attention scores. This further prevents them from addressing the I/O bottleneck of Video-LLMs.

%% file: sec/4_method.tex
\section{Design of ParallelVLM}
\label{sec:method}

\subsection{UV-Prune: Unbiased Verifier-Guided Pruning}
Instead of asking ``which tokens does the model attend to'', we turn to answer ``\emph{which tokens become increasingly aligned with the text query as information flows through the target model layers?}'' Tokens whose alignment with the textual query increases across layers are likely carrying task-essential semantics; conversely, tokens whose alignment stays flat or decreases are more likely redundant for the current query. This suggests a new pruning strategy that captures the semantics rather than raw attention scores. Hence, we utilize vision-text similarity variations of the token representations between the shallow layers as the criterion. Specifically, the vision-text similarity is defined as the cosine similarity between the prefilled video tokens $V_i$ and text tokens $X_j$ at the pre-defined early layer set $l \in \{1,...,L\}$ from the target model:
\begin{equation}
\small
\mathcal{S}_{ij} = \frac{V_i\cdot X_j}{\Vert{V_i}\Vert\Vert{X_j}\Vert}.
\label{similarity}
\end{equation}
We then compute the \emph{similarity variation} between consecutive layers and sum over all text tokens,
\begin{equation}
\small
\overline{\Delta S_i}= \sum_{j=1}^n\sum_{l=1}^L(\mathcal{S}_{ij}^{l}-\mathcal{S}_{ij}^{l-1}).
\label{average_similarity}
\end{equation}
A large positive $\overline{\Delta S_i}$ indicates that video token $V_i$ has gained cross-modal relevance through the network. Given a pruning ratio $\alpha\ (0<\alpha<1)$, we keep the Top-K tokens,
\begin{equation}
\small
V^* = \mathbf{TopK}(\overline{\Delta S_1},...,\overline{\Delta S_m}).
\label{topk_selection}
\end{equation}
As shown in Fig.~\ref{fig:architecture}, the similarity variations are signals measured within the target model during its prefilling stage. Thus, pruning the draft model is directly guided by the target’s own semantics, which provides robust alignment and retains high acceptance rates. Unlike attention scores which bias toward early and query-adjacent tokens, our method reflects a knowledge transfer from the target to the draft model, without being constrained by absolute positions. This enables the natural retention of mid-video frames if there is relevance, preserving both spatial and temporal coherence. 

\subsection{ParallelVLM Pipeline}
The ParallelVLM framework executes SD in two fully parallel stages with an overview shown in Fig.~\ref{fig:architecture}.

\noindent\textbf{Stage I: Parallel Prefilling (PP)}

\noindent\ding{202}\textbf{ Execution Model.} We launch two independent processes: 1) the target process prefills the full video token sequence $V_{1:m}$ with the target model $M_p$; 2) the draft process only prefills the pruned video token sequence $V^{*}$ with the draft model $M_q$. Formally, we have:
\begin{equation}
\small
\text{Parallel}\begin{cases}
\text{Draft: } \text{KV}_q\leftarrow M_q(V^{*}, X_{1:n}) \\
\text{Target: } \text{KV}_p\leftarrow M_p(V_{1:m},X_{1:n}),
\end{cases}\nonumber
\label{eq:parallel_prefilling}
\end{equation}
in which $X_{1:n}$ are the text tokens. Since the target prefilling latency is long, all draft-side operations including target-guided pruning, draft prefilling, and startup token generation can be completed within this span. This completely hides the draft prefilling time and removes the sequential prefilling bottleneck in vanilla SD. 

\noindent\ding{203}\textbf{ Verifier-Guided Draft Model Pruning.}
Once the target process produces intermediate results of the pre-defined early layers, it broadcasts the target model's layer-wise token representations to the draft process. The draft process then executes UV-Prune with pruning ratio $\alpha$ to select the top $1-\alpha$ fraction of video tokens most relevant from the \emph{target’s} perspective,
\begin{equation}
\small
V^{*} = \mathbf{UV}\text{-}\mathbf{Prune}(M_p(V_{1:m},X_{1:n}),\alpha).
\label{Prune}
\end{equation}

\noindent\ding{204}\textbf{ Startup Draft Tokens.} The remaining idle time before target prefilling completes is used by the draft process to generate an initial speculative window of $\gamma$ tokens. These startup tokens are immediately available for verification when Stage I ends, which guarantees that Stage II begins without a warm-up delay. 

\noindent\textbf{Results of Stage I:} After parallel prefilling, (i) the draft model is ready to generate tokens rapidly using the pruned video context. (ii) the target model holds full-context KV caches for lossless verification. (iii) A startup window (e.g., $\gamma$ tokens) is already queued for the first verification round.

\noindent\textbf{Stage II: Parallel Decoding (PD)}

\noindent\ding{205}\textbf{ Execution Model.} The draft and target models operate in a coordinated pipeline: in decoding round $i$ with accepted prefix of length $k$, the draft model $M_q$ generates a window of $\gamma$ candidate tokens $\hat{X}_{k+\gamma+1:k+2\gamma}$ from the pruned context $V^{*}$; simultaneously, the target model $M_p$ verifies $\gamma$ tokens drafted in round $i-1$ using the full context $V_{1:m}$,
\begin{equation}
\small
\text{Parallel}\begin{cases}
\text{Draft}_i:\quad X_{k+
\gamma+t}\sim q(\cdot|V^{*},X_{1:k+\gamma+t-1})\\
\text{Verify}_{i-1}:X_{k+1:k+\gamma}\sim p(\cdot|V_{1:m},X_{1:k})
\end{cases}\nonumber
\label{eq:parallel_decode}
\end{equation}
For $t=1,...,\gamma$, each candidate token $\hat{X}_{k+\gamma+t}$ is accepted with rejection sampling illustrated in Sec.~\ref{subsec:pre_sd}. If a draft token fails verification, sampling resumes from the target distribution for that position. Then the pipeline performs pre-rollback and restart drafting~\cite{Liu2024ParallelSD}.

\noindent\ding{206}\textbf{ Window Size $\gamma$ Selection.} The optimal draft window size $\gamma$ is determined by the pruned speed ratio $c^{*}$:
\begin{equation}
\small
\gamma  = c^{*} = \frac{T_p}{\ \ \ T_q(\alpha)}.
\label{Prefill}
\end{equation}
$T_q(\alpha)$ is the draft decoding time at pruning ratio $\alpha$ and $T_p$ is the target's verification time with full contexts. As an example, without pruning: $T_q=78$ ms, $T_p=420$ ms, $c=T_p/T_q\approx5$, so the optimal window size is $\gamma=5$; with pruning ratio $\alpha=0.9$, $T_q=47$ ms, $T_p=420$ ms, $c=T_p/T_q\approx9$, yielding $\gamma=9$, which is $1.8\times$ larger. 

\noindent\ding{207}\textbf{ Acceptance Rate $\tau$ Dependency.} The actual speedup depends on the pruning quality through acceptance rate $\tau$, $\tau  = \tau(\mathcal{P},\alpha)$. A well-algined pruning strategy $\mathcal{P}$ keeps $\tau$ high even at aggressive pruning ratios ($\alpha= 0.9$).

\subsection{Theoretical Speedup Analysis}
\label{theory}
ParallelVLM accelerates Video-LLMs by overlapping the draft and verification stages, and further addressing the visual redundancy of the draft model with the target's alignment. If the original draft/target speed ratio is $c$ and the average acceptance rate is $\tau$, the theoretical speedup vs. autoregressive decoding is analyzed by the following theorem.
\begin{theorem}[Speedup Ratio of ParallelVLM]
\label{thm:paravlm_speedup}
Denote the video token pruning method as $\mathcal{P}$ and the pruning ratio as $\alpha \in [0,1]$. ParallelVLM overlaps the draft and verification stages, corresponding to an enlarged draft/target speed ratio $c^{*}(\alpha)>c$, and a robust average acceptance rate $\hat{\tau}(\mathcal{P},\alpha)$. For the draft window size of $\gamma=c^{*}(\alpha)$, the speedup ratio of ParallelVLM is,
\begin{eqnarray}
\small
\mathcal{V}_{\mathrm{Vi}PSD} = \frac{\hat{\tau}(\mathcal{P},\alpha)\cdot\gamma\cdot T_p}{\gamma\cdot T_q(\alpha)}=\hat{\tau}(\mathcal{P},\alpha)\cdot c^{*}(\alpha).
\label{psd_speedup}
\end{eqnarray}
\end{theorem}

With the optimal acceptance rate $\hat{\tau}(\mathcal{P},\alpha)=1$, $\mathcal{V}_{\mathrm{Vi}PSD}=c^{*}(\alpha)$ represents $\frac{c^{*}(\alpha)}{c}\times$ speedup than Parallel SD~\cite{Liu2024ParallelSD}. For example, the speedup is $1.8\times$ for LLaVA-OV-7B/72B and $1.6\times$ for Qwen2.5-VL-7B/32B. We refer to the Appendix for detailed proofs and a comprehensive comparison with vanilla SD. 


%% file: sec/5_experiment.tex
\begin{table*}[t]
\caption{Comparison of lossless speculative decoding methods on multiple video understanding tasks. The default window size $\gamma$ of the baselines is set to 5, except for STD~\cite{Zhang2024SparseVLMVT} with $\gamma=9$. ParallelVLM achieves adaptive draft length~\cite{Liu2024ParallelSD}, and we dynamically adjust the $\gamma$ according to the speed ratio $c=T_p/T_q$ to mitigate mutual waiting. From up to bottom, $\gamma=\left\{3,9,5,2,2\right\}$  at the pruning ratio $\alpha=0.9$.}
\label{tab:lossless_comparison}
\centering
\resizebox{\linewidth}{!}{%
\begin{tabular}{l l cc cc cc cc cc c}
\toprule
\multirow{2}{*}{Models} & \multirow{2}{*}{Methods} & \multicolumn{2}{c}{VideoDetailCaption} & \multicolumn{2}{c}{VideoMME} & \multicolumn{2}{c}{MVBench} & \multicolumn{2}{c}{MVLU} & \multicolumn{2}{c}{LongVideoBench}&\multirow{2}{*}{Avg.}\\
\cmidrule(lr){3-4} \cmidrule(lr){5-6} \cmidrule(lr){7-8} \cmidrule(lr){9-10} \cmidrule(lr){11-12} & &$M$ & Speedup & $M$& Speedup & $M$ & Speedup & $M$ & Speedup & $M$ & Speedup\\
\midrule
\multirow{6}{*}{\makecell[c]{LLaVA-OV\\(0.5B \& 7B)} }
& Vanilla SD &  4.26& 1.08$\times$&  4.34& 1.10$\times$ &  4.18 & 1.05 $\times$ & 4.22 & 1.06$\times$ & 4.32& 1.01$\times$& 1.06$\times$  \\
& OnSD & 3.73&  1.40$\times$&  3.67&  1.38$\times$&  3.78& 1.32$\times$ &  3.56&  1.29$\times$&  3.76& 1.36$\times$&  1.35$\times$\\
& SD-Tree & 5.08& 1.28$\times$  &  5.05&  1.28$\times$ & 4.97 & 1.25$\times$  & 4.90 & 1.24$\times$  & 5.12 & 1.26$\times$& 1.26$\times$\\
& SpecVLM &  4.89&  1.83$\times$&  4.83&  1.82$\times$& 4.76 & 1.77$\times$ &  4.71&  1.76$\times$& 4.92 & 1.86$\times$& 1.81$\times$\\
&\cellcolor{pink!30}ParallelVLM & \cellcolor{pink!30}\textbf{8.47} & \cellcolor{pink!30} \textbf{2.12$\times$}& \cellcolor{pink!30}\textbf{7.86} & \cellcolor{pink!30}\textbf{2.06$\times$} & \cellcolor{pink!30}\textbf{8.64}  & \cellcolor{pink!30}\textbf{2.15$\times$}&
\cellcolor{pink!30}\textbf{8.24} & \cellcolor{pink!30}\textbf{2.14$\times$}&
\cellcolor{pink!30}\textbf{8.35} & \cellcolor{pink!30}\textbf{2.10$\times$}&
\cellcolor{pink!30}\textbf{2.11$\times$}   \\

\midrule
\multirow{6}{*}{\makecell[c]{LLaVA-OV\\(7B \& 72B)} }
& Vanilla SD   & 4.14 &  1.98$\times$&  4.26&  2.06$\times$&  3.98&  1.92 $\times$&  4.06& 2.01$\times$& 4.27& 2.10$\times$& 2.01$\times$\\
& OnSD    &  3.67 &  2.14$\times$&  3.65& 2.27$\times$&  3.48&  2.09$\times$&  3.69& 2.23$\times$&  3.74& 2.13$\times$&  2.17$\times$\\
& SD-Tree &  4.64 & 2.18$\times$& 4.62 &  2.33$\times$&  4.70 &  2.26$\times$& 4.53 &  2.17$\times$&  4.78& 2.36$\times$&  2.26$\times$\\
& SpecVLM & 4.51 & 2.76$\times$ &  4.49& 2.80$\times$ & 4.46&  2.68$\times$&  4.26&  2.62$\times$& 4.39& 2.82$\times$ & 2.74$\times$\\
&\cellcolor{pink!30}ParallelVLM & \cellcolor{pink!30}\textbf{6.82}& \cellcolor{pink!30}\textbf{3.43$\times$} & \cellcolor{pink!30}\textbf{6.85} & \cellcolor{pink!30}\textbf{3.38$\times$} & \cellcolor{pink!30}\textbf{6.81} & \cellcolor{pink!30}\textbf{3.28$\times$} & \cellcolor{pink!30}\textbf{6.72} & \cellcolor{pink!30}\textbf{3.24$\times$} & \cellcolor{pink!30}\textbf{6.93} & \cellcolor{pink!30}\textbf{3.46$\times$} & \cellcolor{pink!30}\textbf{3.36$\times$}\\

\midrule
\multirow{6}{*}{\makecell[c]{Qwen2.5-VL\\(7B \& 32B)} }
& Vanilla SD  & 3.37 & 1.26$\times$ &  3.44&  1.24$\times$&  3.38&  1.23$\times$&  3.21&  1.23$\times$&  3.47& 1.31$\times$& 1.25$\times$\\
& OnSD   & 2.76 & 1.38$\times$ &  2.84&  1.32$\times$ &  2.67&  1.27$\times$ &  2.69&  1.34$\times$ &  2.89 & 1.36$\times$ &  1.33$\times$\\
& SD-Tree & 4.21 & 1.60$\times$ & 4.32 & 1.54$\times$ & 4.17 & 1.53$\times$ & 4.09 &  1.52$\times$&  4.23& 1.55$\times$&  1.55$\times$\\
& SpecVLM & 4.18 & 2.11$\times$ &  4.31&  2.13$\times$&  4.12& 2.02$\times$& 4.06 & 2.13$\times$ & 4.28& 2.18$\times$& 2.11$\times$ \\
&\cellcolor{pink!30}ParallelVLM&\cellcolor{pink!30}\textbf{4.23} &\cellcolor{pink!30}\textbf{2.40$\times$}&\cellcolor{pink!30}\textbf{4.37}&\cellcolor{pink!30}\textbf{2.42$\times$}&\cellcolor{pink!30}\textbf{4.24}&\cellcolor{pink!30}\textbf{2.41$\times$}&\cellcolor{pink!30}\textbf{4.15}&\cellcolor{pink!30}\textbf{2.41$\times$}&\cellcolor{pink!30}\textbf{4.43}&\cellcolor{pink!30}\textbf{2.45$\times$}&\cellcolor{pink!30}\textbf{2.42$\times$}\\
\midrule
\midrule
\multirow{4}{*}{\makecell[c]{LLaVA-OV\\(7B \& 7B)} }
& OnSD    &  3.60 & 0.87$\times$ & 3.64 & 0.89$\times$&  3.57&  0.84$\times$&  3.28& 0.82$\times$ & 3.12 & 0.78$\times$& 0.84$\times$\\
& SpecVLM &  4.94&  1.21$\times$& 4.92&  1.22$\times$& 5.06&  1.23$\times$&  4.74&  1.18$\times$&  4.85&  1.25$\times$&  1.22$\times$\\
& STD &  8.04&  1.25$\times$&  8.21& 1.28$\times$ &  7.71&  1.21$\times$&  8.16&  1.26$\times$&  8.04& 1.22$\times$& 1.24$\times$\\
&\cellcolor{pink!30}ParallelVLM & \cellcolor{pink!30}\textbf{14.8}& \cellcolor{pink!30}\textbf{1.53$\times$} & \cellcolor{pink!30}\textbf{15.6} & \cellcolor{pink!30}\textbf{1.55$\times$} & \cellcolor{pink!30}\textbf{13.2} & \cellcolor{pink!30}\textbf{1.49$\times$}& \cellcolor{pink!30}\textbf{15.7} & \cellcolor{pink!30}\textbf{1.59$\times$} & \cellcolor{pink!30}\textbf{15.0}& \cellcolor{pink!30}\textbf{1.57$\times$} & \cellcolor{pink!30}\textbf{1.55$\times$}\\
\midrule
\multirow{4}{*}{\makecell[c]{Qwen2.5-VL\\(7B \& 7B)} }
& OnSD   &  3.39&  0.82$\times$&  2.98& 0.78$\times$ &  3.34& 0.83$\times$ & 3.12 & 0.74$\times$ & 3.27&  0.80$\times$ & 0.79$\times$\\
& SpecVLM & 4.80&  1.17$\times$&  4.91&  1.20$\times$&  4.93&  1.18$\times$&  4.98&  1.25$\times$&  4.85& 1.24$\times$& 1.21$\times$\\
& STD &  8.03& 1.22$\times$ & 7.89 & 1.25$\times$ & 8.11 & 1.24$\times$ &  8.02& 1.18$\times$ & 8.23 & 1.27$\times$& 1.23$\times$\\
&\cellcolor{pink!30}ParallelVLM & \cellcolor{pink!30}\textbf{14.4} & \cellcolor{pink!30}\textbf{1.52$\times$}&\cellcolor{pink!30}\textbf{14.9} & \cellcolor{pink!30}\textbf{1.53$\times$}& \cellcolor{pink!30}\textbf{14.8}& \cellcolor{pink!30}\textbf{1.47$\times$}& \cellcolor{pink!30}\textbf{14.6}& \cellcolor{pink!30}\textbf{1.50$\times$}& \cellcolor{pink!30}\textbf{15.3}& \cellcolor{pink!30}\textbf{1.53$\times$}&\cellcolor{pink!30}\textbf{1.51$\times$}\\
\midrule
\bottomrule
\end{tabular}
}
\end{table*}

\section{Experiments}
\label{sec:experiments}
\subsection{Experimental Setting}
\noindent\textbf{Baseline Methods.} For Video-LLMs, autoregressive decoding (AR) is the default decoding setting ($1.00\times$ baseline). We evaluate the proposed ParallelVLM against a series of \emph{lossless} speculative decoding methods and \emph{lossy} visual token pruning methods. For lossless speculative decoding, we include five methods, which are all plug-and-play with no need of model training.
\begin{itemize}
\item[$\bullet$] Vanilla SD~\cite{Chen2023AcceleratingLL}: Standard implementation where a draft model generates $\gamma$ tokens for target verification.  
\item[$\bullet$] OnSD~\cite{Gagrani2024OnSD}: Light-weight text-only LLM as the draft model for VLM speculative decoding.
\item[$\bullet$] SD-Tree~\cite{Li2024EAGLESS}: Speculative decoding with a 5-layer draft tree structure for multi-branch candidate paths.
\item[$\bullet$] SpecVLM~\cite{Ji2025SpecVLMES}: Verifier's attention-guided draft token pruning for improved Video-LLM speculative decoding.
\item[$\bullet$] STD~\cite{Zhang2025SparsetoDenseAF}: KV cache sparsification of the dense target as the draft model for Video-LLM speculative decoding.
\end{itemize}

For visual token pruning, we collect four dynamic methods that employ the intra-LLM information. We include one-stage importance ranking methods FastV~\cite{FastVChen2024AnII}, PDrop~\cite{Xing2024PyramidDropAY}, and two-stage merge-then-sparsify methods SparseVLM~\cite{Zhang2024SparseVLMVT}, DyCoke~\cite{Tao2024DyCokeD}. Unlike lossless SD methods, pruning is directly performed on the target model.

\noindent\textbf{Benchmarks and Metrics.} Since ParallelVLM focuses on long‑paragraph decoding, we evaluate its performance in video captioning and video description tasks. The benchmarks we use include VideoDetailCaption~\cite{VideoDetailCaption}, VideoMME~\cite{Fu2024VideoMMETF},  MVBench~\cite{Li2023MVBenchAC}, MVLU~\cite{Zhou2024MLVUAC} and LongVideoBench~\cite{Wu2024LongVideoBenchAB}. For metrics, we assess the acceleration effects with: (i) speedup ratio vs. the autoregressive baseline; (ii) the mean accepted length $M$ to reflect the speculation accuracy; (iii) token-wise acceptance ratio $\mathcal{A}$ to measure the distribution shifts of lossy visual pruning methods. Due to the lack of long, detailed ground truth description labels, we leverage the target model with full context for evaluation.

\noindent\textbf{Implementation Details.} Recall that the acceleration of SD mainly depends on the speedup ratio $c$ and alignment between the draft/target models. For comprehensive evaluation, we evaluate five draft/target pairings across LLaVA-OneVision~\cite{Li2024LLaVAOneVisionEV} and Qwen2.5-VL~\cite{Bai2025Qwen25VLTR}, covering $c$ from $1-5$ (rounded to nearest integer): (i) LLaVA-OV ($0.5$B \& $7$B, $c=2$), (ii) LLaVA-OV ($7$B \& $72$B, $c=5$), (iii) Qwen2.5-VL ($7$B \& $32$B: $c=3$), (iv) LLaVA-OV ($7$B \& $7$B, $c=1$),  (v) Qwen2.5-VL ($7$B \& $7$B: $c=1$). For LLaVA-OneVision, we sample $128$ frames for each video with a total of $196 \times 128 = 25,088$ video tokens; for Qwen2.5-VL, we keep the same number of video tokens accordingly. The max output length is set to $K=512$. All the experiments are conducted with 8$\times$ L40S GPUs.

\begin{table*}[t]
\caption{Comparison of lossy visual token pruning methods on multiple video understanding tasks. ParallelVLM achieves a superior speedup ratio losslessly. (The minor drops stem from the inherent uncertainty of LLMs and the precision limits of hardware numerics.)}
\label{tab:lossy_methods}
\centering
\resizebox{\linewidth}{!}{
\begin{tabular}{l cc cc cc cc cc cc}
\toprule
\multirow{2}{*}{Methods} & \multicolumn{2}{c}{VideoDetailCaption} & \multicolumn{2}{c}{VideoMME} & \multicolumn{2}{c}{MVBench} & \multicolumn{2}{c}{MVLU} & \multicolumn{2}{c}{LongVideoBench} & \multicolumn{2}{c}{Avg.}\\
\cmidrule(lr){2-3} \cmidrule(lr){4-5} \cmidrule(lr){6-7} \cmidrule(lr){8-9} \cmidrule(lr){10-11} \cmidrule(lr){12-13}
 & $\mathcal{A}$ & Speedup & $\mathcal{A}$ & Speedup & $\mathcal{A}$ & Speedup & $\mathcal{A}$ & Speedup & $\mathcal{A}$ & Speedup & $\mathcal{A}$ & Speedup\\
\midrule
\multicolumn{13}{c}{LLaVA-OV-72B, \emph{Average Retention Ratio=10\%}} \\
\midrule
FastV    & 83.2\% & 1.60$\times$ & 85.7\% & 1.65$\times$ & 84.1\% & 1.62$\times$ & 86.3\% & 1.66$\times$ & 83.7\% & 1.62$\times$ & \cellcolor{gray!20}84.6\% & \cellcolor{gray!20}1.63$\times$\\
SparseVLM   & 88.5\% & 1.58$\times$ & 86.2\% & 1.61$\times$ & 87.9\% & 1.59$\times$ & 85.8\% & 1.63$\times$ & 87.1\% & 1.60$\times$ & \cellcolor{gray!20}87.1\% & \cellcolor{gray!20}1.60$\times$\\
P-Drop & 86.5\% & 1.55$\times$ & 88.9\% & 1.59$\times$ & 89.2\% & 1.57$\times$ & 87.3\% & 1.60$\times$ & 87.1\% & 1.59$\times$ & \cellcolor{gray!20}87.8\% & \cellcolor{gray!20}1.58$\times$\\
DyCoke & 92.1\% & 1.44$\times$ & 91.5\% & 1.48$\times$ & 89.7\% & 1.46$\times$ & 90.3\% & 1.49$\times$ & 90.9\% & 1.47$\times$ & \cellcolor{gray!20}90.9\% & \cellcolor{gray!20}1.47$\times$\\
\cellcolor{pink!30}ParallelVLM & \cellcolor{pink!30}\textbf{98.2\%} & \cellcolor{pink!30}\textbf{3.43$\times$} & \cellcolor{pink!30}\textbf{99.7\%} & \cellcolor{pink!30}\textbf{3.38$\times$} & \cellcolor{pink!30}\textbf{99.5\%} & \cellcolor{pink!30}\textbf{3.28$\times$} &
\cellcolor{pink!30}\textbf{99.1\%} & \cellcolor{pink!30}\textbf{3.24$\times$} & \cellcolor{pink!30}\textbf{98.9\%} & \cellcolor{pink!30}\textbf{3.46$\times$} &
\cellcolor{pink!30}\textbf{99.1\%} &
\cellcolor{pink!30}\textbf{3.36$\times$}\\
\midrule
\multicolumn{13}{c}{Qwen2.5-VL-32B, \emph{Average Retention Ratio=10\%}}  \\
\midrule
FastV    & 83.7\% & 1.51$\times$ & 81.9\% & 1.53$\times$ & 80.8\% & 1.55$\times$ & 82.6\% & 1.56$\times$ & 82.5\% & 1.55$\times$ & \cellcolor{gray!20}82.3\% & \cellcolor{gray!20}1.54$\times$\\
SparseVLM   & 86.9\% & 1.48$\times$ & 84.3\% & 1.51$\times$ & 85.7\% & 1.49$\times$ & 83.8\% & 1.52$\times$ & 86.3\% & 1.50$\times$ & \cellcolor{gray!20}85.4\% & \cellcolor{gray!20}1.50$\times$\\
P-Drop & 83.2\% & 1.49$\times$ & 85.9\% & 1.53$\times$ & 84.1\% & 1.51$\times$ & 86.4\% & 1.52$\times$ & 84.0\% & 1.55$\times$ & \cellcolor{gray!20}84.7\% & \cellcolor{gray!20}1.52$\times$\\
DyCoke & 88.9\% & 1.42$\times$ & 86.7\% & 1.45$\times$ & 87.3\% & 1.43$\times$ & 89.1\% & 1.46$\times$ & 85.9\% & 1.44$\times$ & \cellcolor{gray!20}87.6\% & \cellcolor{gray!20}1.44$\times$\\
\cellcolor{pink!30}ParallelVLM & \cellcolor{pink!30}\textbf{98.8\%} & \cellcolor{pink!30} \textbf{2.40$\times$}& \cellcolor{pink!30}\textbf{98.2\%} & \cellcolor{pink!30}\textbf{2.42$\times$} & \cellcolor{pink!30}\textbf{98.6\%} & 
\cellcolor{pink!30}\textbf{2.41$\times$} & \cellcolor{pink!30}\textbf{99.2\%} & \cellcolor{pink!30}\textbf{2.41$\times$} & \cellcolor{pink!30}\textbf{98.7\%} & \cellcolor{pink!30}\textbf{2.45$\times$} &
\cellcolor{pink!30}\textbf{98.7\%} &
\cellcolor{pink!30}\textbf{2.42$\times$}\\
\bottomrule
\end{tabular}
}
\end{table*}

\subsection{Main Results}
\noindent\textbf{Comparison with Lossless SD Methods.} While model combinations with different speed ratios $c$ greatly influence the overall acceleration, we observe three advantages of ParallelVLM over the baseline methods. 
\textbf{(i) ParallelVLM achieves significantly better acceleration across the five combinations.} As shown in Tab.~\ref{tab:lossless_comparison}, ParallelVLM achieves 2.11$\times$, 3.36$\times$ and 2.42$\times$ acceleration on average for LLaVA-OV-$0.5$B/$7$B, LLaVA-OV-$7$B/$72$B, Qwen2.5-VL-$7$B/$32$B, respectively. Notably, ParallelVLM surpasses the existing SOTA SpecVLM~\cite{Ji2025SpecVLMES} by $0.30\sim0.64\times$, mainly attributed to the co-design of the parallel pipeline and UV-Pruning. \textbf{(ii) ParallelVLM enables adaptive draft length $M$ with continuous draft stages.} Unlike vanilla SD framework that requires sequential waiting for draft and verification, ParallelVLM facilitates continuous drafting and simultaneous verification with adaptive draft length~\cite{Liu2024ParallelSD}. As shown in Tab.~\ref{tab:lossless_comparison}, ParallelVLM maintains a considerable mean accepted length ($M$) of $8.31, 6.83, 4.28$ for the three combinations. The robust acceptance reduces unsuccessful speculation and rollbacks~\cite{Shen2025SpeculativeDV}, ensuring the overall speedup performance. \textbf{(iii) ParallelVLM makes Self-SD more effective in practice.} Self-SD~\cite{Zhang2025SparsetoDenseAF} is a special paradigm with perfect alignment traded by a smaller speed ratio $c$. While vanilla SD methods wait for the draft model to finish, ParallelVLM verifies candidate tokens while drafting, greatly reducing unnecessary waiting time. As shown in the last two rows in Tab.~\ref{tab:lossless_comparison}, Self-SD achieves a mean accepted length of $M > 14$, with an acceleration of $1.51\sim1.55\times$, which accounts for $\sim0.30\times$ higher speedup than SpecVLM~\cite{Ji2025SpecVLMES} and STD~\cite{Zhang2025SparsetoDenseAF} methods.

\noindent\textbf{Comparison with Lossy Visual Token Pruning Methods.} Compared with vanilla visual token pruning methods~\cite{FastVChen2024AnII,Zhang2024SparseVLMVT,Xing2024PyramidDropAY,Tao2024DyCokeD}, ParallelVLM is favorable from the two following aspects. \textbf{(i) ParallelVLM has a theoretically lossless guarantee}. As shown in Tab.~\ref{tab:lossy_methods}, although FastV~\cite{FastVChen2024AnII}, SparseVLM~\cite{Zhang2024SparseVLMVT}, P-Drop~\cite{Xing2024PyramidDropAY}, DyCoke~\cite{Tao2024DyCokeD} can reduce video tokens without a draft-then-verify framework, they inevitably lead to performance degradation ($-9.1\%\sim-17.7\%$). Therefore, it is questionable for their practical applications, particularly in tasks demanding fine-grained details. In contrast, ParallelVLM presents little information loss with target model verification. \textbf{(ii) ParallelVLM offers significantly higher speedup.} While we also observe an alignment-speedup trade-off among visual token pruning methods in Tab.~\ref{tab:lossy_methods}, they only provide $1.44\sim1.64\times$ decoding time acceleration. However, ParallelVLM delegates rapid proposals to the 7B draft model, delivering $3.36\times$ and $2.42\times$ decoding acceleration for $72$B and $32$B models, respectively.

\subsection{Ablation Study}
\noindent\textbf{Effect of Pruning Ratio $\alpha$.} Tab.~\ref{tab:c_alpha} and Fig.\ref{fig:ablation_speedup_alpha} confirm that higher draft pruning ratios substantially increase the speed ratio in Eq. \eqref{Prefill}. As shown in Tab.~\ref{tab:c_alpha}, with the original speedup ratios of $c=\left\{5,3\right\}$ for LLaVA-OV-7B/72B and Qwen2.5-VL-7B/32B, the speed ratio increases to $c^{*}=\left\{7,4\right\}$ at $\alpha=0.5$ and $c^{*}=\left\{9,5\right\}$ at $\alpha=0.9$, respectively. In ParallelVLM,  a larger $c$ expands the optimal window size $\gamma$ we can choose within a target model verification time. As shown in Fig.~\ref{fig:ablation_speedup_alpha}, the overall speedup performance increases from $2.94\times$ to $3.36\times$ for LLaVA-OV-7B/72B and from $2.08\times$ to $2.42\times$ when $\alpha$ increases from $0.0$ to $0.9$. The slight drop in $M$ at high $\alpha$ could be offset by the gains of larger $\gamma$. However, extreme pruning ($\alpha=1.0$) breaks alignment and confirms the need for a balanced setting. We observe that $\alpha=0.9$ strikes a reasonable balance between the expanded window size and model alignment.
\begin{table}[t]
    \centering
    \caption{The speed ratio $c=T_p/T_q$ with different pruning ratio $\alpha$. The forward time is reported in \emph{ms}.}
    \label{tab:c_alpha}
    \resizebox{0.48\textwidth}{!}{%
    \begin{tabular}{lcccc}
    \toprule
    \multirow{2}{*}{Models} & \multirow{2}{*}{Target $T_p$ } & \multicolumn{3}{c}{Draft $T_q$} \\
    \cmidrule(lr){3-5}
    & & $\alpha=0.0$ & $\alpha=0.5$ & $\alpha=0.9$\\
    \toprule
    \multirow{2}{*}{\makecell[c]{LLaVA-OV\\(7B \& 72B)} }
    & \multirow{2}{*}{\makecell[c]{420 ($\pm8.64$)}} & 78.3 ($\pm3.46$)  & 57.5 ($\pm2.38$) & 46.6 ($\pm1.52$) \\ 
    & &\cellcolor{gray!20}$c=5$ & \cellcolor{cyan!10}$c=7$ & \cellcolor{pink!30}$c=9$ \\
    \midrule
    \multirow{2}{*}{\makecell[c]{Qwen2.5-VL\\(7B \& 32B)} }
    & \multirow{2}{*}{\makecell[c]{203 ($\pm6.89$)}} & 63.7 ($\pm2.42$)  & 48.2 ($\pm1.66$) & 39.8 ($\pm1.24$) \\ 
    & &\cellcolor{gray!20}$c=3$ & \cellcolor{cyan!10}$c=4$ & \cellcolor{pink!30}$c=5$ \\
    \bottomrule
    \end{tabular}
    }
\end{table}
\begin{figure}[t]
    \centering
    \includegraphics[width=0.48\textwidth]{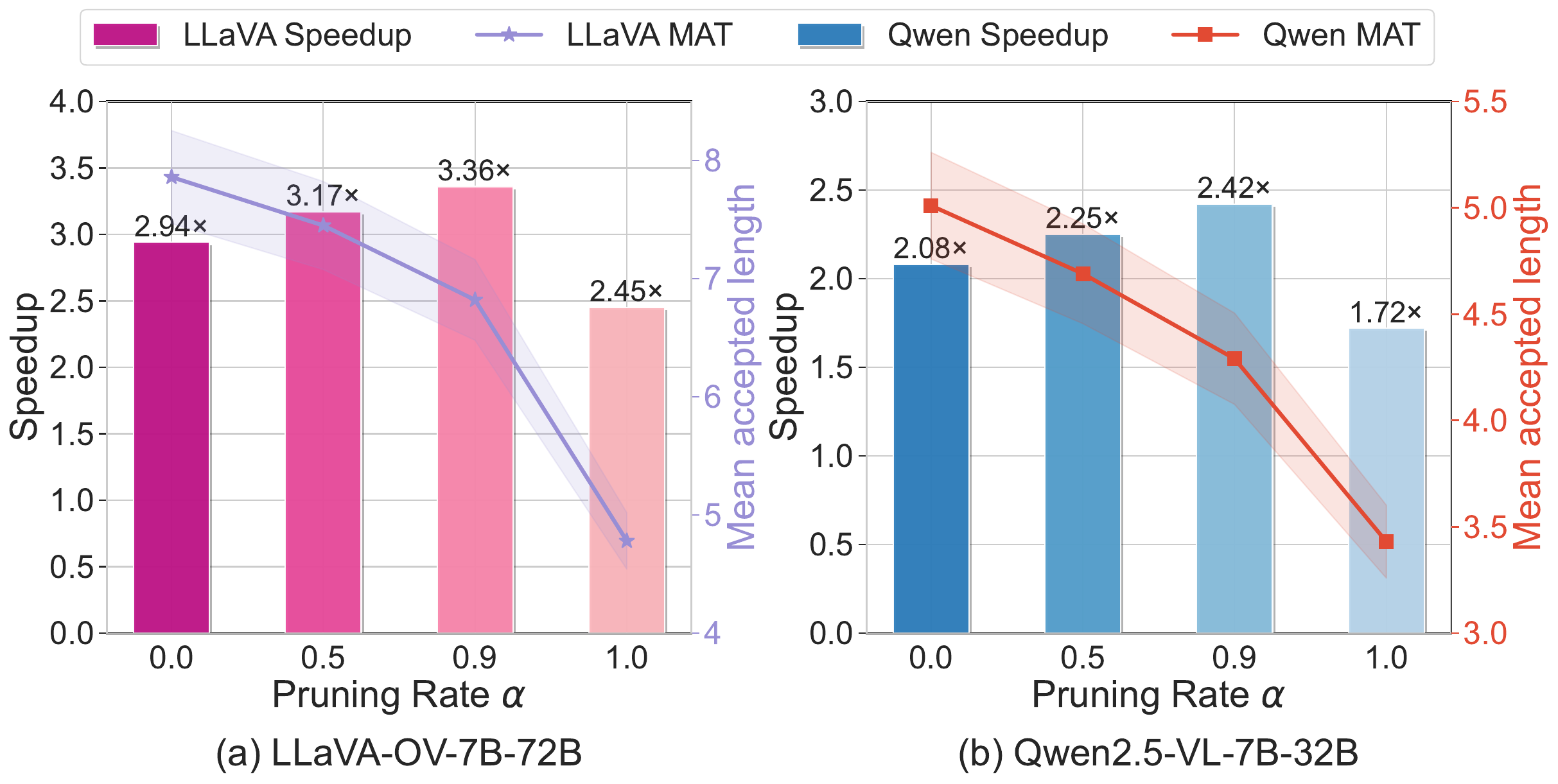}
    \caption{Ablation study of the pruning ratio $\alpha$. ParallelVLM strikes a trade-off between speedup and acceptance at $\alpha=0.9$.}
    \label{fig:ablation_speedup_alpha}
    \vspace{-0.15in}
\end{figure}

\begin{figure*}
    \centering
    \begin{subfigure}[b]{0.48\textwidth}  
        \centering
        \includegraphics[width=\textwidth]{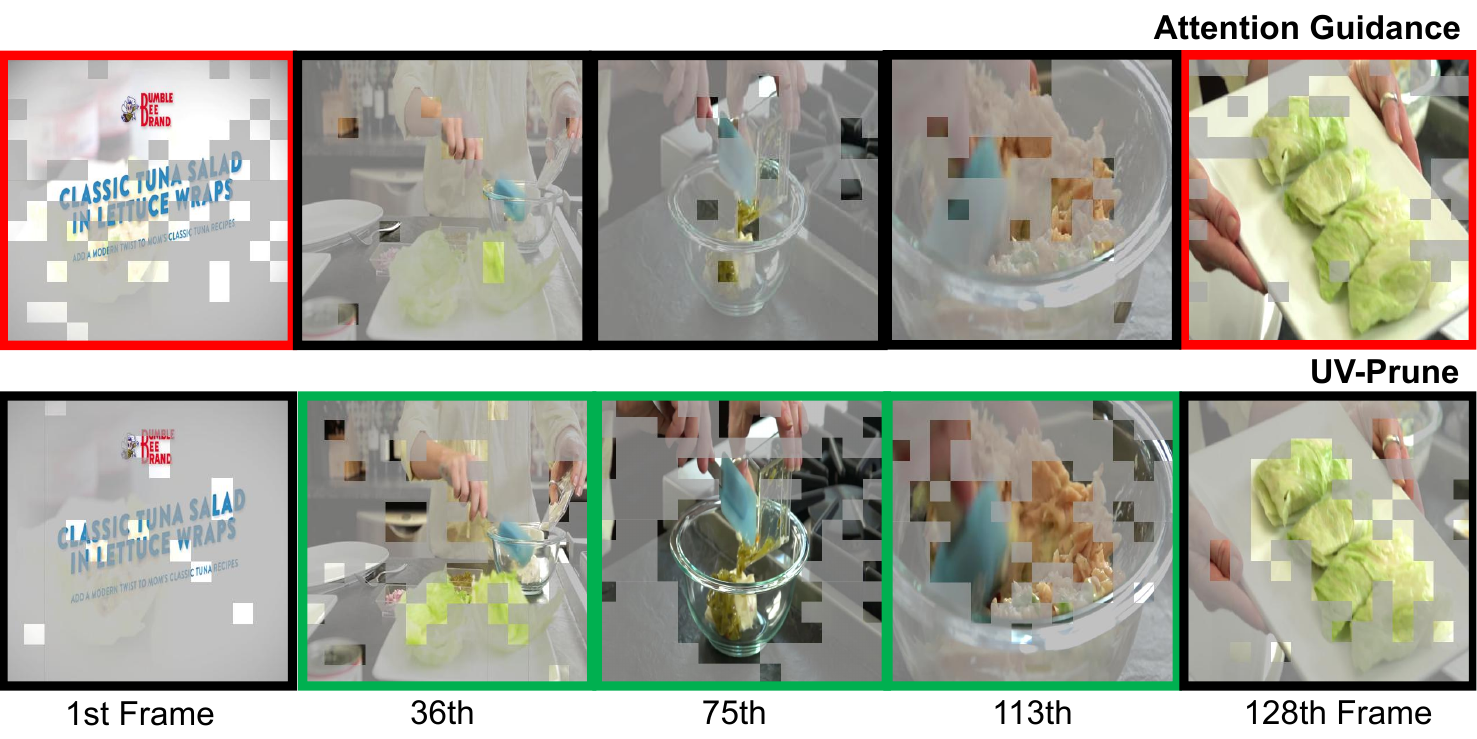}
    \end{subfigure}
    \hspace{0.02\textwidth}  
    \begin{subfigure}[b]{0.48\textwidth}  
        \centering
        \includegraphics[width=\textwidth]{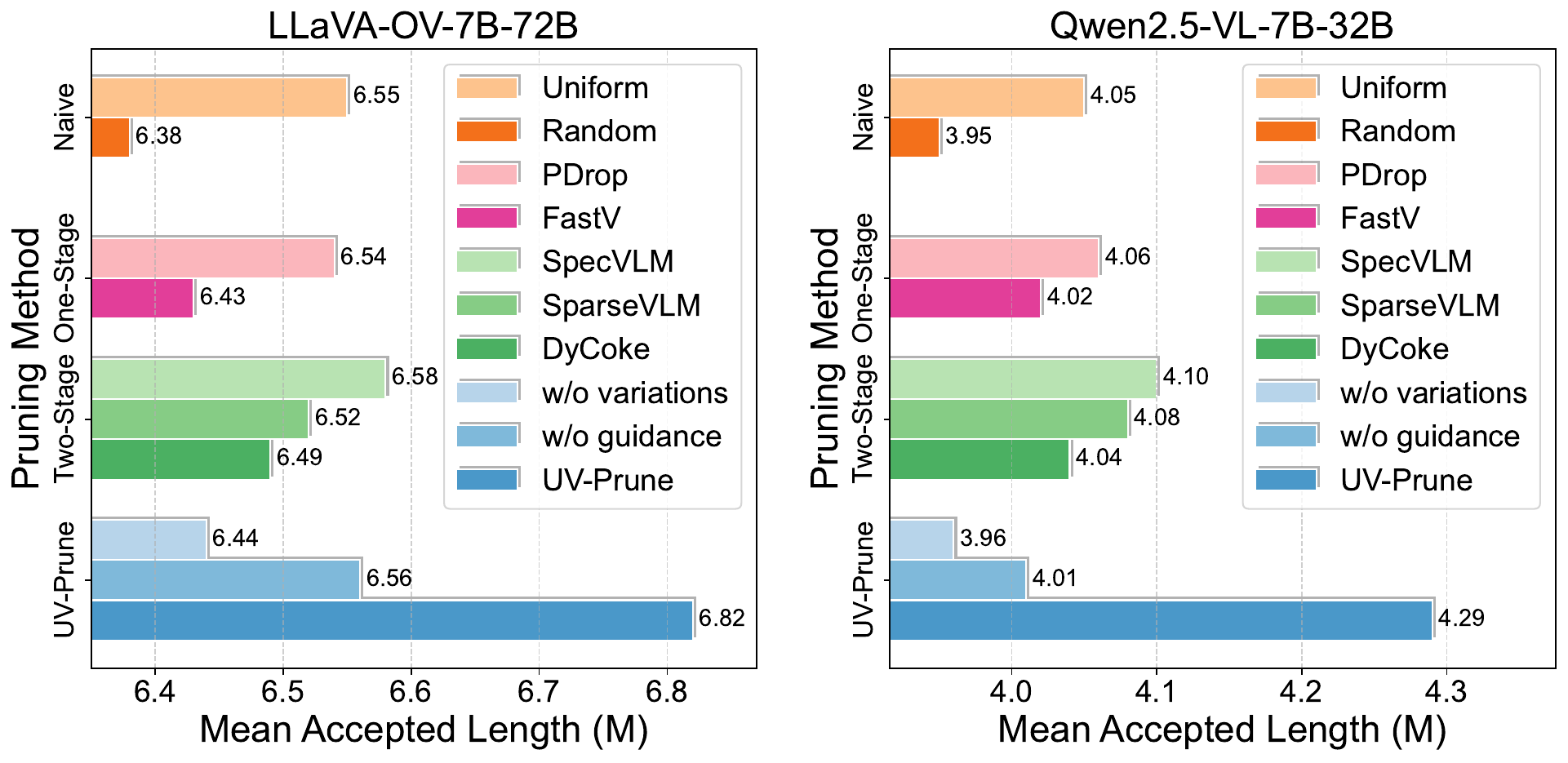}
    \end{subfigure}
    \vspace{-0.1in}
    \caption{\small{Ablation Study of UV-Prune. (a) Visualizations of selected tokens across frames. (b) Impact of UV-Prune's verifier-guidance and similarity-variation. With the joint effort of the two components, ParallelVLM achieves robust acceptance, exceeding existing baselines.}}
    \label{fig:ablation_uv_prune}
    \vspace{-0.2in}
\end{figure*}

\noindent\textbf{Impact of Unbiased Verifier-Guided Pruning.} As shown in Fig.~\ref{fig:ablation_uv_prune} (a), UV-Prune effectively mitigates the positional bias of attention-guided pruning meanwhile selectively focusing on different visual cues across frames. In Fig.~\ref{fig:ablation_uv_prune} (b), we further evaluate the alignment of different pruning methods and report the mean accepted length $M$ when $\alpha=0.9$. For comparison, although SpecVLM~\cite{Ji2025SpecVLMES} leverages the target model's attention scores for guidance, it provides only partial accurate alignment due to positional bias and $M$ drops by $0.2$. For other methods without verifier guidance, the accept length drops more sharply by $0.3\sim0.5$. In contrast, UV-Prune consistently achieves a higher accepted length.

\noindent\textbf{FlashAttention Compatibility.} Because UV‑Prune does not rely on explicit attention‑score computation, it is fully compatible with FlashAttention~\cite{Dao2023FlashAttention2FA}. Tab.~\ref{tab:flash_attention} demonstrates increasing benefits with longer sequence lengths. At $24$K sequence length, FlashAttention adds $1.35\times$ extra speedup on top of ParallelVLM and at $32$K, the gain further reaches to $1.41\times$. With an increasing demand for long video understanding tasks, the compatibility with efficient attention operators is increasingly important in the future.

\begin{table}[h]
    \centering
    \caption{Integration of flash attention at different sequence lengths. The forward time is reported for LLaVA-OV-72B in \emph{ms}.}
    \label{tab:flash_attention}
    \resizebox{0.48\textwidth}{!}{%
    \small 
    \begin{tabular}{lcccc}
    \toprule
    & \multicolumn{4}{c}{Sequence Length} \\
    \cmidrule(lr){2-5}
    & 8K & 16K & 24K & 32K \\
    \toprule
    w/o flash. & 302 ($\pm 6.15$) & 361 ($\pm 7.06$) & 420 ($\pm 8.32$) & 479 ($\pm 9.14$) \\
    w/ flash.  & 248 ($\pm 4.42$) & 280 ($\pm 4.67$) &  311 ($\pm 5.21$)& 340 ($\pm 6.08$)  \\
    Speedup    & $1.22\times$ & $1.29\times$ & $1.35\times$ & $1.41\times$ \\
    \bottomrule
    \end{tabular}
    }
\end{table}

%% file: sec/6_related_work.tex
\vspace{-0.15in}
\section{Related Work}
\label{sec:related_work}
\textbf{Visual Token Pruning for VLMs.} A plethora of works have identified substantial redundancy in visual tokens and proposed pruning strategies to speed up VLMs. Existing approaches can generally be grouped into pre-LLM~\cite{Alvar2025DivPruneDV,Shang2024LLaVAPruMergeAT, Huang2024PruneVidVT, Shen2025FastVIDDD,wen2025stop,liu2025video,Tang2025AdaptiveKS, yang2025visionzip, Khaki2025SparseVILADV,10.1007/978-3-032-37271-0_22} and intra-LLM pruning~\cite{FastVChen2024AnII,Zhang2024SparseVLMVT,Xing2024PyramidDropAY,Tao2024DyCokeD,zeng2026hybridkvhybridkvcache}. The pre-LLM methods (e.g. PruMerge~\cite{Shang2024LLaVAPruMergeAT} and FastVID~\cite{Shen2025FastVIDDD}) remove redundant tokens at the vision encoder stage before they are passed to the language model, leveraging fixed similarity or clustering criteria. Intra‑LLM methods (e.g. FastV~\cite{FastVChen2024AnII}, SparseVLM~\cite{Zhang2024SparseVLMVT}, PDrop~\cite{Xing2024PyramidDropAY}, DyCoke~\cite{Tao2024DyCokeD}) dynamically prune or sparsify the KV cache during the LLM’s forward computation. Although visual token pruning methods reduce token load in long-video scenarios, they offer limited acceleration and cause inevitable distributional shifts that lead to performance degradation.

\noindent\textbf{Speculative Decoding for LLMs and VLMs.} SD utilizes a lightweight draft model to draft candidate tokens and verify them simultaneously~\cite{Leviathan2022FastIF, Xia2022SpeculativeDE}, thereby reducing the autoregressive latency of the target model. It can be categorized into draft model training-based~\cite{cai2024medusa, du2024glide,Li2024EAGLESS,Li2024EAGLE2FI,Li2025EAGLE3SU} and training-free methods~\cite{fu2024break,Zhao2024OuroborosGL,Chen2023AcceleratingLL}. Subsequent works~\cite{Liu2024ParallelSD, Shen2025SpeculativeDV, shen2026double} further explore parallel designs to eliminate mutual waiting. When extended to VLMs, the draft model has been instantiated as a text-only model~\cite{Gagrani2024OnSD}, a distilled model~\cite{Huang2025SpecVLMFS}, or a lightweight ``twig'' model~\cite{Shao2025GrowingAT}. However, naive adaptations to VLM often overlook the fact that draft latency can grow significantly in long-context video understanding, where large amounts of visual tokens must be processed. As a result, current approaches are often ineffective~\cite{Gagrani2024OnSD} or inefficient~\cite{Huang2025SpecVLMFS, Shao2025GrowingAT} for large-scale video inputs.

\noindent\textbf{Speculative Decoding for Video-LLMs.} The long-context video prefilling contradicts the principle of a lightweight draft model in speculative decoding. To address this challenge, Sparse-to-Dense (STD)~\cite{Zhang2025SparsetoDenseAF} introduces a Self-SD framework and LVSPEC~\cite{ji2026foresttreeslooselyspeculative} explores loosely speculative decoding with visual-semantic guidance. The closest work to ours is SpecVLM~\cite{Ji2025SpecVLMES}, which incorporates SD with attention-guided token pruning from the target model. In contrast to standard SD with sequential dependency, we enables complete parallelism while making video token pruning alignment-aware, unbiased and fully compatible with modern hardware optimizations like FlashAttention~\cite{Dao2023FlashAttention2FA}.

%% file: sec/7_conclusion.tex
\vspace{-0.05in}
\section{Conclusion}

In this paper, we propose ParallelVLM to remove autoregressive decoding bottlenecks in Video-LLMs caused by massive video token inputs. It combines parallel prefilling and parallel decoding with UV-Prune, an attention-free, verifier-guided token reduction method using vision-text similarity variation. By co-designing with parallel SD, our framework improves the speedup ratio while preserving lossless generation, accelerating LLaVA-OV-72B and Qwen2.5-VL-32B by up to $3.36\times$ and $2.42\times$, respectively.

\section*{Acknowledgements}
This work is supported by National Science Foundation of China under grants 62576310, 62394341 and Zhejiang Provincial National Science Foundation of China under Grant No. LZ25F020007.

%% file: sec/X_suppl.tex
\clearpage
\setcounter{page}{1}
\maketitlesupplementary

\renewcommand{\thesection}{\Alph{section}}
\setcounter{section}{0}

For readers who are not familiar with speculative decoding and parallel decoding, we provide more detailed explanations in the appendix. The appendix is organized as follows: experimental details are presented in Sec.~\ref{suppl:impl}, parallel decoding in Sec.~\ref{suppl:pearl}, and theoretical analysis of the proposed method in Sec.~\ref{suppl:theory}. Additionally, we elaborate on the origin of the unbias nature of UV-Prune in Sec.~\ref{suppl:uv-prune}. Finally, we include several illustrative case studies in Sec.~\ref{suppl:case}.

\vspace{-0.01in}
\section{Experimental Details}
\label{suppl:impl}
\noindent\textbf{Tasks and Metrics.} Since speculative decoding~\cite{Leviathan2022FastIF,Xia2022SpeculativeDE} focuses on accelerating long-paragraph decoding, we select five video understanding and description datasets VideoDetailCaption~\cite{VideoDetailCaption}, VideoMME~\cite{Fu2024VideoMMETF}, MVBench~\cite{Li2023MVBenchAC}, MVLU~\cite{Zhou2024MLVUAC}, LongVideoBench~\cite{Wu2024LongVideoBenchAB} with rich semantic information. We filter out videos shorter than one minute and randomly sample 50 clips from each dataset for evaluation. Unlike traditional Visual Question Answering (VQA) formats \footnote{Different from our finer-grained long-paragraph decoding setting, VQA evaluation suites like LMMs-Eval mainly focus on reality checking.}, we employ prompt engineering to generate finer-grained descriptions. For the lossy visual token pruning method, due to the lack of long, detailed ground truth description labels in the community, we save the autoregressive output texts after pruning and use the target model with full context to perform token-wise speculative sampling~\cite{Leviathan2022FastIF} to measure the acceptance rate $\mathcal{A}$. This approach not only saves the cost of manual annotation, but also accurately reflects the distribution shifts of the pruned models.

\noindent\textbf{Model Combinations.} Regarding the model combinations, since ParallelVLM is a training-free method, aligning the draft model with the target model poses a challenge. We therefore select model combinations with different parameter sizes from the same series to ensure that the draft and target models have a basic alignment capability. In experiments we find that models from the LLaVA-OV~\cite{Li2024LLaVAOneVisionEV} series exhibit a higher alignment ($M=8.31$ for 0.5B/7B and $M=6.83$ for 7B/72B), while Qwen2.5-VL~\cite{Bai2025Qwen25VLTR} shows a relatively lower alignment ($M=4.28$ for 7B/32B). Additionally, the 0.5B/7B combination within the LLaVA-OV series achieves better alignment than the 7B/72B combination, as the former has more similar model capacities and fine-tuning data scales. Although our approach is totally training-free, draft training-based methods~\cite{cai2024medusa,Li2024EAGLESS,Li2025EAGLE3SU} can be seamlessly integrated to achieve improved alignment and higher acceleration ratios in future work.

\noindent\textbf{Window Size $\gamma$ Selection.} For simplicity of explanation in the main text, we directly rounded the speed ratio between the draft model and the target model $c=T_p/T_q$ to an integer value. In practice, their acceleration ratio is not strictly an integer. In Tab.~\ref{tab:c}, we present the exact acceleration ratios for all the five combinations. Rounding the window size up or down relative to the speed ratio $c$ are both reasonable options, the superiority between which depends on the relative magnitudes of the window size $\gamma$ and the mean accepted length $M$~\cite{Liu2024ParallelSD}. For combinations of LLaVA-OV-7B/72B and Qwen2.5-VL-7B/32B with aggressive window sizes and relatively conservative accepted lengths, rounding down to reduce mutual waiting is a wiser choice. For Self-SD with narrow window sizes and high accepted lengths, rounding up is a more sensible option.

\begin{table}[t]
\caption{Sensitivity of speed ratio $c=T_p/T_q$ with the increase of pruning ratio $\alpha$. The forward time is reported in \emph{ms}.}
\label{tab:c}
\centering
\resizebox{\linewidth}{!}{%
\begin{tabular}{lcccc}
\toprule
\multirow{2}{*}{Models} & \multirow{2}{*}{Target $T_p$ } & \multicolumn{3}{c}{Draft $T_q$} \\
\cmidrule(lr){3-5}
& & $\alpha=0.0$ & $\alpha=0.5$ & $\alpha=0.9$\\
\toprule
\multirow{2}{*}{\makecell[c]{LLaVA-OV\\(0.5B \& 7B)} }
& \multirow{2}{*}{\makecell[c]{78.3 ($\pm3.46$) }} & 46.1 ($\pm1.21$)  & 32.5 ($\pm0.89$) & 25.8 ($\pm0.32$) \\ 
& &\cellcolor{gray!20}$c=1.7$ & \cellcolor{cyan!10}$c=2.4$ & \cellcolor{pink!30}$c=3.0$ \\
\midrule
\multirow{2}{*}{\makecell[c]{LLaVA-OV\\(7B \& 72B)} }
& \multirow{2}{*}{\makecell[c]{420 ($\pm8.64$)}} & 78.3 ($\pm3.46$)  & 57.5 ($\pm2.38$) & 46.6 ($\pm1.52$) \\ 
& &\cellcolor{gray!20}$c=5.3$ & \cellcolor{cyan!10}$c=7.3$ & \cellcolor{pink!30}$c=9.0$ \\
\midrule
\multirow{2}{*}{\makecell[c]{Qwen2.5-VL\\(7B \& 32B)} }
& \multirow{2}{*}{\makecell[c]{203 ($\pm6.89$)}} & 63.7 ($\pm2.42$)  & 48.2 ($\pm1.66$) & 39.8 ($\pm1.24$) \\ 
& &\cellcolor{gray!20}$c=3.2$ & \cellcolor{cyan!10}$c=4.2$ & \cellcolor{pink!30}$c=5.1$ \\
\midrule
\multirow{2}{*}{\makecell[c]{LLaVA-OV\\(7B \& 7B)} }
& \multirow{2}{*}{\makecell[c]{78.3 ($\pm3.46$)}} & 78.3 ($\pm3.46$)  & 57.5 ($\pm2.38$) & 46.6 ($\pm1.52$) \\ 
& &\cellcolor{gray!20}$c=1.0$ & \cellcolor{cyan!10}$c=1.4$ & \cellcolor{pink!30}$c=1.7$ \\
\midrule
\multirow{2}{*}{\makecell[c]{Qwen2.5-VL\\(7B \& 7B)} }
& \multirow{2}{*}{\makecell[c]{63.7 ($\pm2.42$)}} & 63.7 ($\pm2.42$)  & 48.2 ($\pm1.66$) & 39.8 ($\pm1.24$) \\ 
& &\cellcolor{gray!20}$c=1.0$ & \cellcolor{cyan!10}$c=1.3$ & \cellcolor{pink!30}$c=1.6$ \\
\bottomrule
\end{tabular}
}
\end{table}

\begin{table}[t]
\centering
\caption{Peak memory usage (GB) and latency breakdown of LLaVA-OV-7B/72B (s). The output length is set to 512.}
\label{tab:flash}
\small 
\begin{tabular}{lccc}
\toprule
\multirow{2}{*}{\makecell[c]{Operation}} & \multicolumn{3}{c}{Method} \\
\cmidrule(lr){2-4}
 & Autoregressive & SD & ParallelVLM \\
\midrule
Peak Memory & 198.85 & 225.19 & 225.34 \\
\midrule
Draft Prefill & 7.92 & 7.92 & / \\
Target Prefill & 44.23 & 44.23 & 44.23 \\
Draft Decode & / & 58.14 & / \\
Target Decode & 228.14 & 54.85 & 67.85 \\
\midrule
Token/s & 2.25 & 4.53 & 7.55 \\
\midrule
Total Latency & 280.29 & 165.15 & 112.08 \\
\bottomrule
\end{tabular}
\vspace{-0.10in}
\end{table}

\begin{figure*}
    \centering
\includegraphics[width=1.0\linewidth]{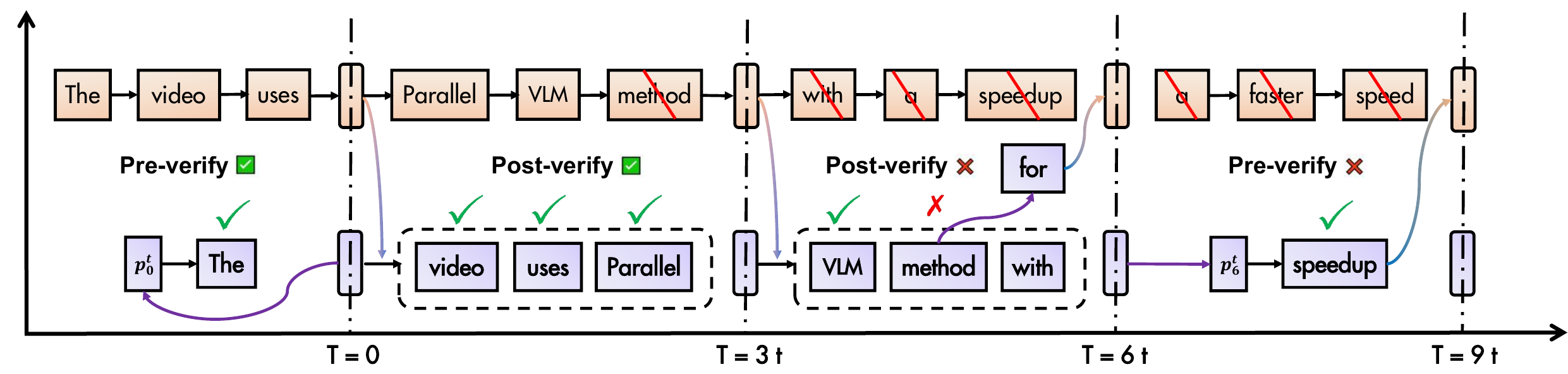}
\vspace{-0.2in}
    \caption{\small{An illustrative example of decoding ``\texttt{The video uses ParallelVLM for speedup}" is provided to demonstrate the verification mode switching of the parallel pipeline, as well as its advancement and rollback mechanisms. Here $t$ denotes the one single forward time of the draft model and the window size $\gamma$ is set to 3.}}
    \label{fig:para_sd_case}
\vspace{-0.15in}
\end{figure*}

\noindent\textbf{Memory Consumption and Latency Breakdown.} We report the peak memory usage and latency breakdown in Tab.~\ref{tab:flash}. Although performing the draft/target forward functions simultaneously, ParallelVLM incurs no more memory footprint beyond vanilla SD with draft model pruning ($+26$ GB than autoregressive). Furthermore, ParallelVLM directly saves the draft model’s prefilling time with parallel prefilling ($-7.92$ s), and significantly accelerates the decoding speed through the visual alignment aware parallelism. Our approach achieves a notable $7.55$ token/s throughput, which is $3.36\times$ faster than autoregressive decoding. While the throughput metric may vary on different devices, the speedup ratio will persist for methodological superiority.

\noindent\textbf{Tree-Attention Integration.} Applying a draft tree structure to speculative decoding can help improve the accepted length, but it is not feasible under a parallel architecture. The draft tree generates KV caches for multiple candidates, which constitutes a batch of additional overhead. This overhead grows exponentially with the depth $h$  of the tree during parallelization (even sparse trees~\cite{Li2024EAGLESS} exhibit considerable growth). In vanilla speculative decoding, the depth of the draft tree is generally equal to the window size $h=\gamma$. However, for parallel decoding, the depth of the draft tree depends on the adaptive draft length~\cite{Liu2024ParallelSD} and may increase $h$ by multiple times. To address the limitation of applying the tree structure to parallel decoding, SpecBranch~\cite{Shen2025SpeculativeDV} attempts to achieve a balance with a sparser ``branch prediction" structure, which we anticipate will find broader applications in the future.

\section{Parallel Pipeline Details}
\label{suppl:pearl}

For clarity, we collectively denote the prefix of visual-text tokens $V_{1:m},X_{1:n}$ as $\mathbf{x}$ here to describe the pipeline advancement and rollback mechanisms of parallel decoding.

As illustrated in Algorithm.~\ref{alg:speculative_decoding}, Parallel SD~\cite{Liu2024ParallelSD} sets two verification modes for the target model: (i) Pre-verify, which verifies the first draft token in advance, and (ii) Post-verify, which performs batch verification meanwhile generating new draft tokens. We elaborate on more details in the following with the illustrative example in Fig.~\ref{fig:para_sd_case} of decoding the sentence ``\texttt{The video uses ParallelVLM for speedup}" with the window size $\gamma=3$.

\begin{algorithm}[H]
\caption{Parallel Decoding Pipeline.}
\label{alg:speculative_decoding}
\begin{algorithmic}[1]             
\Require draft model $M_q$, target model $M_p$, input prefix $\mathbf{x}$, output length $K$, window size $\gamma$.
\State $\text{mode} \gets \text{``Pre-verify''}$   
\While{$\text{len}(\mathbf{x}) < K$}
    \If{$\text{mode} = \text{``Pre-verify''}$}
        \State $(\mathbf{x}, \text{mode}) \gets \text{Pre-verify}(M_q, M_p, \mathbf{x}, \gamma)$
    \Else
        \State $(\mathbf{x}, \text{mode}) \gets \text{Post-verify}(M_q, M_p, \mathbf{x}, \gamma)$
    \EndIf
\EndWhile
\end{algorithmic}
\end{algorithm}
\vspace{-0.15in}

At $T=0$, both the draft model and the target model complete pre-filling, they output a bonus token and execute the initial Pre-verify. The Pre-verify verifies the first token ``\texttt{The}" before the verification stage. Upon successful matching (accept ``\texttt{The}"), the mode switches to Post-verify. From $T=0\sim3t$, a window size of draft tokens ``\texttt{video uses Parallel}" is waiting to be verified. The Post-verify checks them with batch parallelism, accepts them all, and maintains the Post-verify mode. From $T=3t\sim6t$, another window size of draft tokens ``\texttt{VLM method with}" is waiting to be verified. However, ``\texttt{method}" is rejected by the target model, so a new token ``\texttt{for}" is resampled from the target's distribution. The KV cache of the draft model rollbacks to ``\texttt{VLM}" and starts from ``\texttt{for}" to draft new tokens, meanwhile the mode is reset back to Pre-verify. From $T=6t\sim9t$, since there are no draft tokens left, the target model verifies the first draft token ``\texttt{a}". Unfortunately, the target model rejects it and thinks that the resampled ``\texttt{speedup}" is a better choice. The KV cache of the draft model thus rollbacks to ``\texttt{for}" and restarts from ``\texttt{speedup}" to continue its drafting. The whole pipeline from $T=0\sim9t$ thus outputs ``\texttt{The video uses ParallelVLM for speedup}" within only three target model verification times. The maximum adaptive draft length~\cite{Liu2024ParallelSD} here is 5, as ``\texttt{The video uses ParallelVLM}" is continuous without interruption of the first two draft stages.

\section{Theoretical Analysis}
\label{suppl:theory}
Let $T_q=t$ be the single forward time of the draft model and $T_p=c\cdot t$ be the verification time of the target model. For a predefined hyperparameter $\gamma$ to control window size, we derive the ideal and practical speedup ratio by calculating the per-token time, respectively. 

\begin{theorem}[Vanilla SD (Ideal)]\label{thm:sd} 
Under full acceptance of the $\gamma$ draft tokens, the per-token time of vanilla SD is:
\begin{eqnarray}
\small
T_{SD}=\frac{\gamma\cdot T_q+T_p}{\gamma+1}=\frac{\gamma+c}{\gamma+1}\cdot t.
\label{sd}
\end{eqnarray}
\end{theorem}

\noindent Thus, the ideal speedup ratio is $\mathcal{V}_{SD}=\frac{T_p}{T_{SD}}=c\cdot\frac{\gamma+1}{\gamma+c}$. 

\begin{theorem}[Parallel SD (Ideal)]
\label{thm:para_sd} Under full acceptance of the $\gamma$ draft tokens, the per-token time of parallel SD is:
\begin{eqnarray}
\small
T_{PSD}=\frac{\rm{max}(\gamma t, c t)}{\gamma}=\begin{cases} t, \ \ \ \quad\gamma\ge c\\ \frac{c}{\gamma}t,\quad \gamma<c\end{cases}
\label{para_sd}
\end{eqnarray}
\end{theorem}

\noindent When $\gamma=c$, the ideal speedup ratio is $\mathcal{V}_{PSD}=\frac{T_p}{T_{PSD}}=c$, representing $c\times$ speedup versus autoregressive decoding and ${\mathcal{V}_{PSD}}/{\mathcal{V}_{SD}}=\frac{\gamma+c}{\gamma+1}=\frac{2c}{c+1}\rightarrow2$$\times$ $(c\gg1)$ speedup compared to vanilla SD. Note that $\mathcal{V}_{PSD}$ can enjoy a higher gain of $c^*>c$ with our draft model pruning setting. 

\begin{theorem}[Vanilla SD (Practical)]\label{thm:sd_practical} 
Under partial acceptance rate $\tau\ (0<\tau<1)$ of $\gamma$ draft tokens, the per-token time of vanilla SD is:
\begin{eqnarray}
\small
T_{SD}=\frac{\gamma\cdot T_q+T_p}{\tau\cdot\gamma+1}=\frac{\gamma+c}{\tau\cdot\gamma+1}\cdot t.
\label{sd_practical}
\end{eqnarray}
\end{theorem}

\noindent Thus, the ideal speedup ratio is $\mathcal{V}_{SD}=\frac{T_p}{T_{SD}}=c\cdot\frac{\tau\cdot\gamma+1}{\gamma+c}$, which increases monotonically with the acceptance rate $\tau$.

\begin{theorem}[Parallel SD (Practical)]
\label{thm:para_sd_practical}Under partial acceptance rate $\tau\ (0<\tau<1)$ of $\gamma$ draft tokens, the per-token time of parallel SD is:
\begin{eqnarray}
\small
T_{PSD}=\frac{\mathrm{max}(\gamma t, c t)}{\tau\cdot\gamma}=\begin{cases} \frac{1}{\tau}t, \quad\gamma\ge c\\ \frac{c}{\tau\cdot\gamma}t,\quad \gamma<c\end{cases}
\label{para_sd_practical}
\end{eqnarray}
\end{theorem}

\noindent When $\gamma=c$, the practical speedup ratio is $\mathcal{V}_{PSD}=\frac{T_p}{T_{PSD}}=\tau\cdot c$, which also increases monotonically with the acceptance rate $\tau$. Recall that Theorem.~\ref{thm:paravlm_speedup} gives $\mathcal{V}_{\text{Vi}PSD}=\hat{\tau}(\mathcal{P},\alpha)\cdot c^*(\alpha)$. It represents an expanded speed ratio $c^{*}(\alpha)>c$ with a robust acceptance rate $\hat{\tau}(\mathcal{P},\alpha)$, where $\mathcal{P}$ is the visual alignment aware pruning and $\alpha$ is the pruning ratio. $\mathcal{V}_{\text{Vi}PSD}$ strikes a balance between $c^{*}(\alpha)$ and $\hat{\tau}(\mathcal{P},\alpha)$ that exceeds $\mathcal{V}_{PSD}$ by up to $0.3\sim0.4\times$ (see Fig.~\ref{fig:ablation_speedup_alpha} (a)).

Although Theorem.~\ref{thm:sd} -~\ref{thm:para_sd_practical} provide general illustrations, the theoretical speedup ratio of parallel decoding is actually more complex in practice. As described in Sec.~\ref{suppl:pearl}, the above theory holds only in Post-verify mode; however, if pre-verify (resulting from rollbacks) is taken into account, analysis must be conducted based on the distribution of accepted lengths. Following SpecBranch~\cite{Shen2025SpeculativeDV}, we assume the random variable of the accepted lengths $X$ follows a truncated geometric distribution $X\sim TruncGeo(\tau,\gamma)$.

\newtheorem*{templemma}{Lemma 1} %

\begin{templemma}[Expected Draft Accepted Length]
\label{lemma:trunc}
\renewcommand{\thelemma}{1} 
For truncated geometric distribution $X\sim TruncGeo(\tau,\gamma)$,
\begin{eqnarray}
\small
E[X]=\frac{\tau(1-{\tau}^{\gamma})}{1-\tau}.
\label{trunc}
\end{eqnarray}
\end{templemma}

\begin{proof} Expand the expected value of random variable $X$,
\begin{eqnarray}
\small
E[X]&=&\sum_{k=0}^{\gamma}k\cdot \mathbb{P}(X=k)\nonumber\\&=&\sum_{k=0}^{\gamma-1}k\cdot(1-\tau){\tau}^{k}+\gamma\cdot{\tau}^{\gamma}.\nonumber
\end{eqnarray}
Let $S=\sum_{k=0}^{\gamma-1}{\tau}^k=\frac{1-{\tau}^k}{1-\tau}$. Take differentiation regarding $\tau$, we have,
\begin{eqnarray}
\small
\frac{\rm{d}S}{\rm{d}\tau}=\sum_{k=0}^{\gamma-1}k{\tau}^{k-1}=\frac{1-\gamma{\tau}^{\gamma-1}+(\gamma-1){\tau}^{\gamma}}{(1-{\tau})^2}\nonumber
\end{eqnarray}
\begin{align*}
\small
E[X]=(1-\tau)\tau\cdot\frac{\rm{d}S}{\rm{d}{\tau}}+\gamma\cdot{\tau}^{\gamma}=\frac{\tau(1-{\tau}^{\gamma})}{1-\tau}\nonumber\tag*{\qedhere}
\end{align*}
\end{proof}

\begin{theorem}[Parallel SD (with Rollback)]
\label{thm:para_sd_rollback} The per-token time of parallel SD under rollback is,
\begin{eqnarray}
\small
T_{PSD_r}=\frac{2\cdot \mathrm{max}(\gamma t,ct)}{(1+{\tau}^{\gamma})\cdot\frac{\tau(1-{\tau}^{\gamma})}{1-\tau}}.
\end{eqnarray}
\end{theorem}

\begin{proof} Define the acceptance vector $\omega=(\omega_1,...,\omega_{\gamma})\in\left\{0,1\right\}^{\gamma}$, where $\omega_i=1$ if and only if token $i$ is accepted. The accepted token count is,
\begin{eqnarray}
\small
X=\sum_{i=1}^{\gamma}\omega_i,\quad\mathbb{P}(\omega_i=1)=\tau\ (\rm{i.i.d.}).\nonumber
\end{eqnarray}
To compute the total number of tokens with retry, define two round of: 1) $\gamma$ tokens (accepted if $\omega=1$); 2) Retry if Round 1 fails, which yields $E[X]$ tokens. Thus, the total expectation is:
\begin{eqnarray}
\small
E_{total}&=&{\tau}^{\gamma}(\gamma+E[X])+(1-{\tau}^{\gamma})\frac{(E[X]-\gamma{\tau}^{\gamma})}{1-{\tau}^{\gamma}}\nonumber\\&=&(1+{\tau}^{\gamma})\cdot E[X]\nonumber
\end{eqnarray}
This implies that Parallel SD (with Rollback) achieves an acceleration of $(1+{\tau}^{\gamma})\times$ compared to vanilla SD (with Rollback). As $\tau\rightarrow1$, the acceleration ratio reaches $2\times$, matching the ideal case in Theorem.~\ref{thm:para_sd}. Thus, we calculate the per-token time with the two rounds:   
\begin{align*}
\small
T_{PSD_r}=\frac{T_{total}}{E_{total}}=\frac{2\cdot\mathrm{max}(\gamma t,ct)}{(1+{\tau}^{\gamma})\cdot\frac{\tau(1-{\tau}^{\gamma})}{1-\tau}}.\nonumber\tag*{\qedhere}
\end{align*}
\end{proof}
\noindent When $\gamma=c$, the speedup ratio of Parallel SD with rollback is $\mathcal{V}_{PSD_r}=\frac{T_p}{T_{PSD_r}}=\frac{1}{2}\sum_{k=1}^{2c}{\tau}^k$. It increases monotonically with the acceptance rate $\tau$ in polynomial order, so maintaining robust acceptance rates is even more crucial for parallel decoding. ParallelVLM increases $c$ by pruning and keeps an unbiased alignment for higher $\tau$, achieving considerable speedup. As shown in Tab.~\ref{tab:pruning_performance}, under the extreme pruning ratio $\alpha\rightarrow1$ where $c^*(\alpha)$ keeps almost the same, we observe the rapid decrease of speedup ratio with lower acceptance rates, confirming the high-sensitivity effect of parallel decoding performance on acceptance rates and the necessity of the visual alignment aware pruning.

\begin{table}[htbp]
  \centering
  \small
  \caption{Experiments with LLaVA-OV-7B/72B on VideoDetailCaption under extremely higher pruning ratio ($\alpha\rightarrow1$).  }
  \begin{tabular}{lcccc}
    \toprule
    \textbf{Pruning ratio $\alpha$} & \textbf{0.90} & \textbf{0.95} & \textbf{0.99} & \textbf{1.00} \\
    \midrule
     Accepted Lengths $M$  &  6.82  &   6.53 &  5.74    &  4.53  \\
     Acceptance rates $\tau$ &  0.41  &  0.37 &   0.33  &  0.27 \\
     \midrule
     Speedup   & 3.43$\times$  & 3.28$\times$  & 2.90$\times$& 2.53$\times$  \\
    \bottomrule
  \end{tabular}
  \label{tab:pruning_performance}
  \vspace{-0.15in}
\end{table}

\section{Why is UV-Prune Essential?}
\label{suppl:uv-prune}
In recent years, numerous studies have focused on pruning VLMs~\cite{FastVChen2024AnII, Xing2024PyramidDropAY, Tao2024DyCokeD, Zhang2024SparseVLMVT} to achieve inference acceleration at the cost of a certain degree of generation quality. In contrast, our work explores the transfer of video-based understanding from more powerful target models to draft models, making it orthogonal to prior research. We retain the full context of the target model to ensure theoretically lossless generation, and perform pruning on the draft model to align with the target model -- marking the first application in parallel decoding framework to allow window size expansion.

Feather~\cite{Endo2024FeatherTT} has discussed the issue of positional bias in the image domain. They observe that the shallow layers of LLMs serve as ``active" regions for visual token comprehension and attribute the attention bias at the vision-text boundaries to the more similar RoPE positional encodings. By computing attention scores without positional encodings as the criterion, they achieve better performance gains.

However, directly applying Feather’s method is no longer feasible for the target model to guide the draft model pruning. The target and draft models differ greatly in embedding dimensions and network depth, leading to architectural mismatch when performing intra-LLM pruning directly on the draft model. In addition, Feather~\cite{Endo2024FeatherTT} requires re-computation of the target model’s attention scores without RoPE, which is not only time-consuming but also conflicts with FlashAttention~\cite{Dao2023FlashAttention2FA}.

To avoid introducing additional computations and ensure compatibility with FlashAttention, we adopt token representations from the shallow layers of the target model as the pruning anchors. Since attention itself is a similarity weight, using cross-modal (vision-text) similarity is a well-motivated approach. To further eliminate the bias induced by positional encodings, we calculate the similarity variations of visual tokens at each position layer by layer to ``offset" the interference of positional encodings. In this way, UV-Prune accurately identifies the few key video tokens prioritized by the target model among a large number of tokens, and effectively transfers this focus to the draft model, achieving salient yet robust image-text relevance. 

As shown in Fig.~\ref{fig:layer}, both variation and similarity enhance the effectiveness of video token selection. The variation mechanism of UV-Prune mitigates ``positional bias", leading to significant improvements for both attention and similarity criterion. We also observe that similarity-variation converges at $L=20$ (and it generalizes well for different model combinations). Thus, we predefine $L=20$ as the number of layers for the target model to guide the draft model’s pruning. The propagation of token representations in such shallow layers consumes negligible transmission time ($\sim0.01$ s) and ensures that the subsequent pre-filling of the draft model is fully hidden within the target model’s pre-filling span. Most importantly, our alignment-aware pruning method meets the requirement for robust acceptance rates specified in Theorem~\ref{thm:para_sd_rollback}.

\begin{figure}[t]
    \centering
\includegraphics[width=1.0\linewidth]{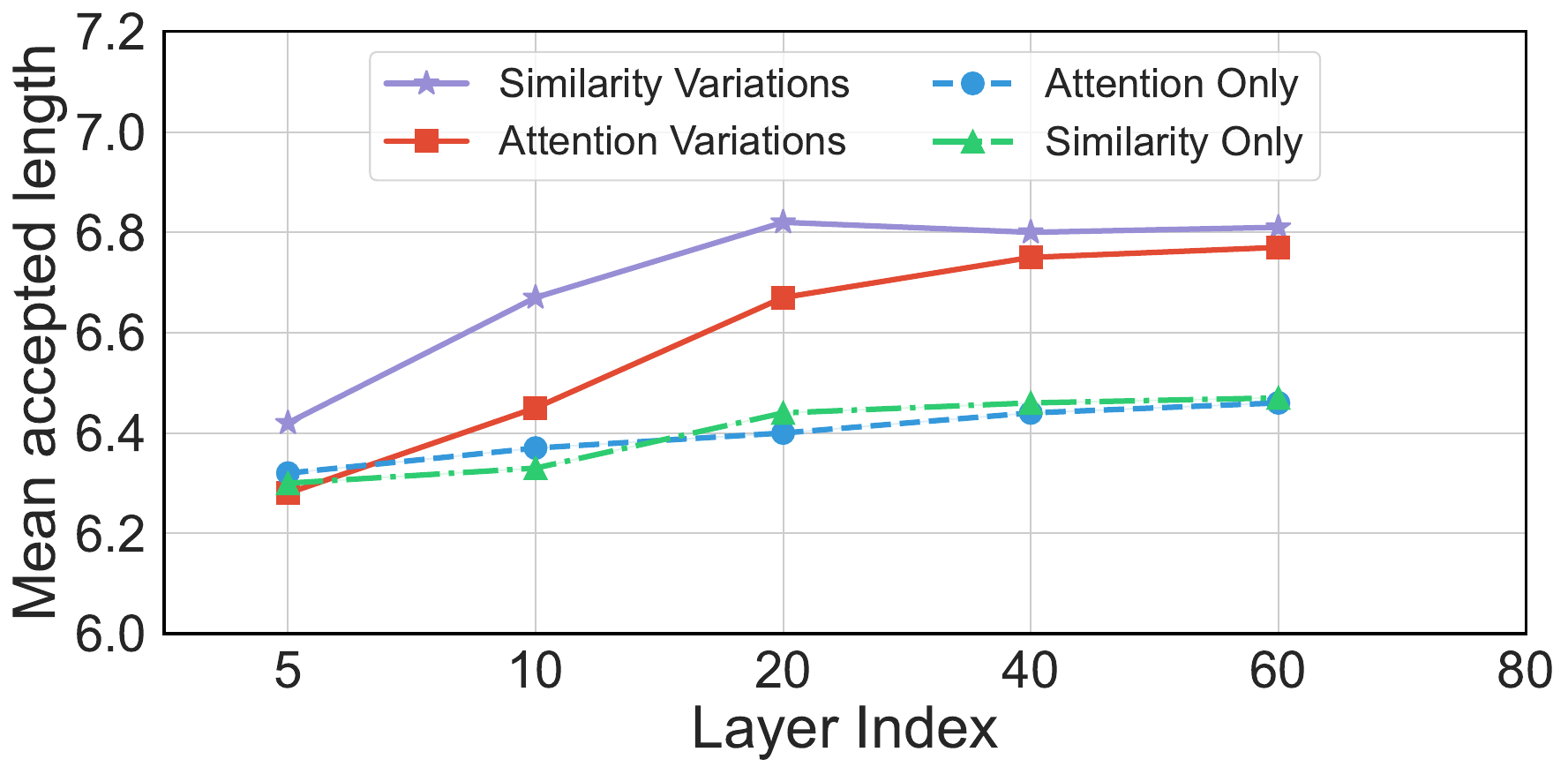}
\vspace{-0.30in}
    \caption{\small{Experiments with LLaVA-OV-7B/72B on VideoDetailCaption with different draft model pruning strategies.}}
    \label{fig:layer}
\vspace{-0.25in}
\end{figure}

\section{Case Study}
\label{suppl:case}
We present two generation paragraphs in Fig.~\ref{fig:case_1} (LLaVA-OneVision-7B/72B) and Fig.~\ref{fig:case_2} (Qwen2.5-VL-7B/32B) to intuitively demonstrate the effects of unbiased pruning guided by the target model in long-sequence video tokens, as well as the lossless generation performance of ParallelVLM compared to lossy visual token pruning schemes. We observe that ParallelVLM maintains finer-grained details and accurately reflects the realities in the videos.

\begin{figure*}
    \centering
\includegraphics[width=1.0\linewidth]{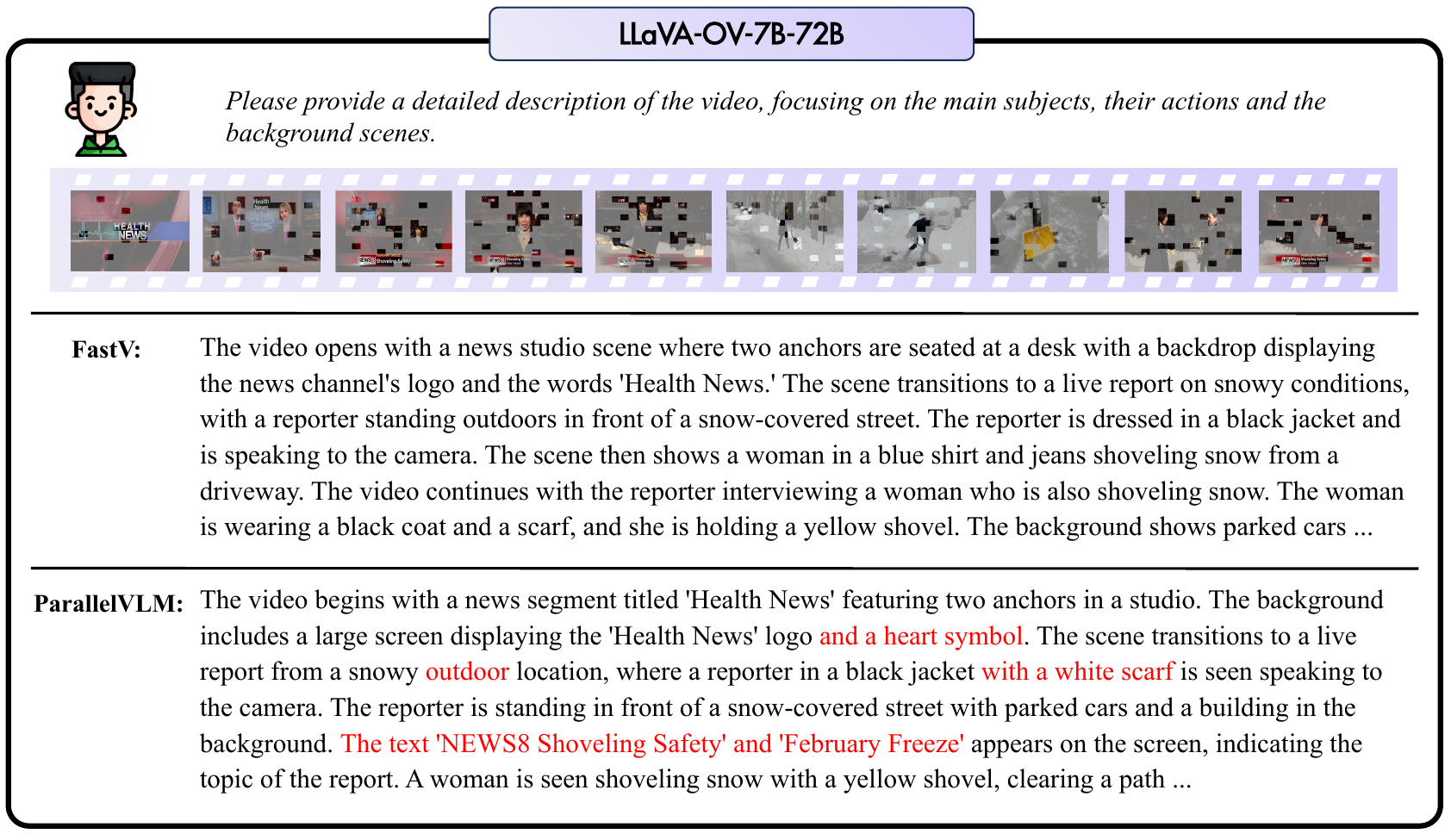}
\vspace{-0.2in}
    \caption{\small{Case study for LLaVA-OV-7B/72B.}}
    \label{fig:case_1}
\vspace{-0.15in}
\end{figure*}

\begin{figure*}
    \centering
\includegraphics[width=1.0\linewidth]{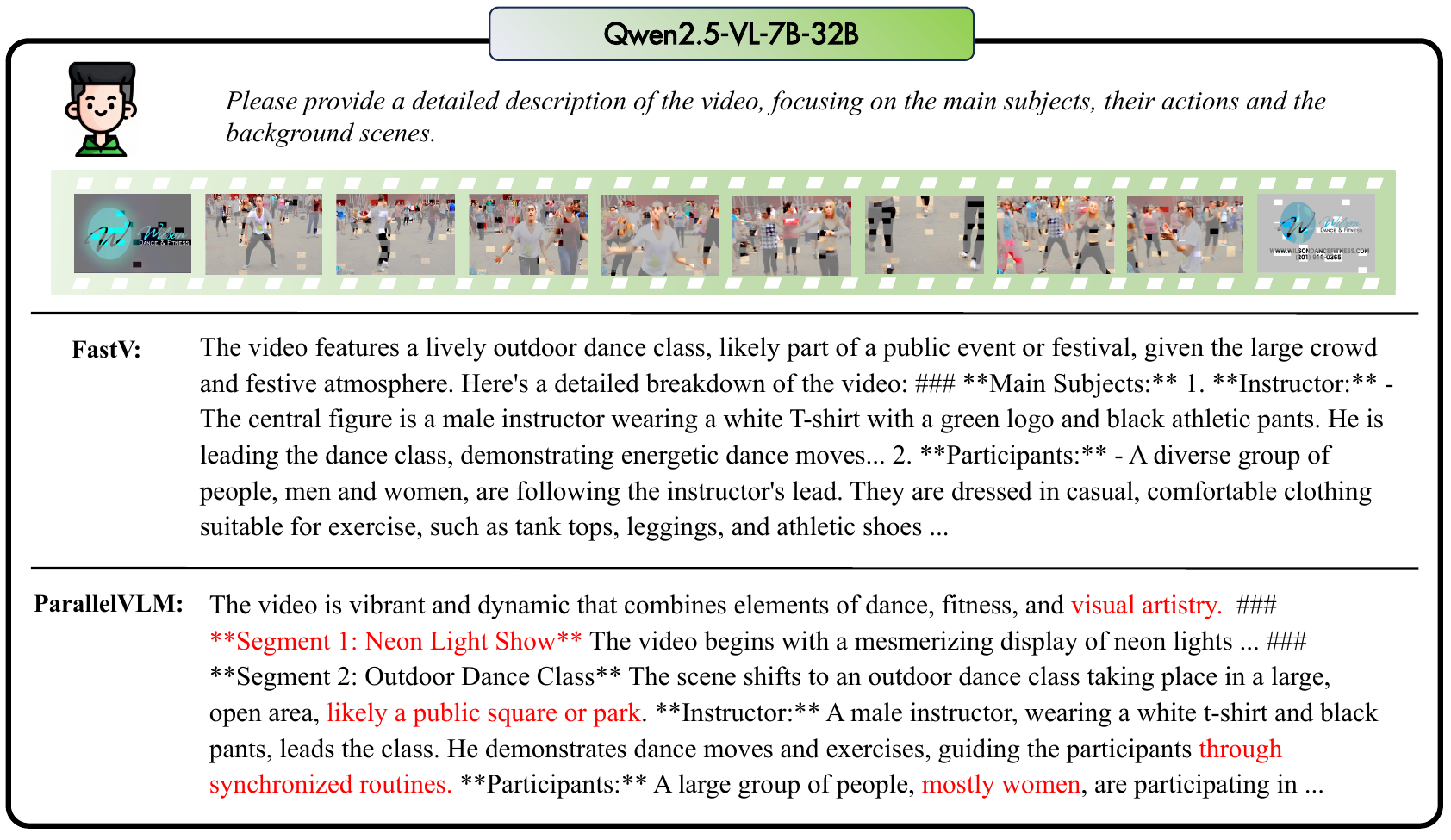}
\vspace{-0.2in}
    \caption{\small{Case study for Qwen2.5-VL-7B/32B.}}
    \label{fig:case_2}
\vspace{-0.15in}
\end{figure*}

%% file: main.bib
@String(ICCV= {Int. Conf. Comput. Vis.})

@String(ECCV= {Eur. Conf. Comput. Vis.})

@String(ICCV  = {ICCV})

@String(ECCV  = {ECCV})

@article{Achiam2023GPT4TR,
  title={Gpt-4 technical report},
  author={Achiam, Josh and Adler, Steven and Agarwal, Sandhini and Ahmad, Lama and Akkaya, Ilge and Aleman, Florencia Leoni and Almeida, Diogo and Altenschmidt, Janko and Altman, Sam and Anadkat, Shyamal and others},
  journal={arXiv preprint arXiv:2303.08774},
  year={2023}
}

@article{DeepSeekAI2024DeepSeekV3TR,
  title={Deepseek-v3 technical report},
  author={Liu, Aixin and Feng, Bei and Xue, Bing and Wang, Bingxuan and Wu, Bochao and Lu, Chengda and Zhao, Chenggang and Deng, Chengqi and Zhang, Chenyu and Ruan, Chong and others},
  journal={arXiv preprint arXiv:2412.19437},
  year={2024}
}

@article{yang2025qwen3,
  title={Qwen3 technical report},
  author={Yang, An and Li, Anfeng and Yang, Baosong and Zhang, Beichen and Hui, Binyuan and Zheng, Bo and Yu, Bowen and Gao, Chang and Huang, Chengen and Lv, Chenxu and others},
  journal={arXiv preprint arXiv:2505.09388},
  year={2025}
}

@article{Touvron2023LLaMAOA,
  title={Llama: Open and efficient foundation language models},
  author={Touvron, Hugo and Lavril, Thibaut and Izacard, Gautier and Martinet, Xavier and Lachaux, Marie-Anne and Lacroix, Timoth{\'e}e and Rozi{\`e}re, Baptiste and Goyal, Naman and Hambro, Eric and Azhar, Faisal and others},
  journal={arXiv preprint arXiv:2302.13971},
  year={2023}
}

@article{Liu2023VisualIT,
  title={Visual instruction tuning},
  author={Liu, Haotian and Li, Chunyuan and Wu, Qingyang and Lee, Yong Jae},
  journal={Advances in neural information processing systems},
  volume={36},
  pages={34892--34916},
  year={2023}
}

@article{Li2024MiniGeminiMT,
  title={Mini-gemini: Mining the potential of multi-modality vision language models},
  author={Li, Yanwei and Zhang, Yuechen and Wang, Chengyao and Zhong, Zhisheng and Chen, Yixin and Chu, Ruihang and Liu, Shaoteng and Jia, Jiaya},
  journal={IEEE Transactions on Pattern Analysis and Machine Intelligence},
  year={2025},
  publisher={IEEE}
}

@article{hurst2024gpt,
  title={Gpt-4o system card},
  author={Hurst, Aaron and Lerer, Adam and Goucher, Adam P and Perelman, Adam and Ramesh, Aditya and Clark, Aidan and Ostrow, AJ and Welihinda, Akila and Hayes, Alan and Radford, Alec and others},
  journal={arXiv preprint arXiv:2410.21276},
  year={2024}
}

@inproceedings{Lin2023VideoLLaVALU,
  title={Video-llava: Learning united visual representation by alignment before projection},
  author={Lin, Bin and Ye, Yang and Zhu, Bin and Cui, Jiaxi and Ning, Munan and Jin, Peng and Yuan, Li},
  booktitle={Proceedings of the 2024 conference on empirical methods in natural language processing},
  pages={5971--5984},
  year={2024}
}

@article{Li2024LLaVAOneVisionEV,
  title={Llava-onevision: Easy visual task transfer},
  author={Li, Bo and Zhang, Yuanhan and Guo, Dong and Zhang, Renrui and Li, Feng and Zhang, Hao and Zhang, Kaichen and Zhang, Peiyuan and Li, Yanwei and Liu, Ziwei and others},
  journal={arXiv preprint arXiv:2408.03326},
  year={2024}
}

@article{Bai2025Qwen25VLTR,
  title={Qwen2.5-VL Technical Report},
  author={Shuai Bai and Keqin Chen and Xuejing Liu and Jialin Wang and Wenbin Ge and Sibo Song and Kai Dang and Peng Wang and Shijie Wang and Jun Tang and others},
  journal={ArXiv},
  year={2025},
  volume={abs/2502.13923},
  url={https://api.semanticscholar.org/CorpusID:276449796}
}

@article{Fang2024MMBenchVideoAL,
  title={Mmbench-video: A long-form multi-shot benchmark for holistic video understanding},
  author={Fang, Xinyu and Mao, Kangrui and Duan, Haodong and Zhao, Xiangyu and Li, Yining and Lin, Dahua and Chen, Kai},
  journal={Advances in Neural Information Processing Systems},
  volume={37},
  pages={89098--89124},
  year={2024}
}

@inproceedings{Liu2025ProtectingYV,
  title={Protecting your video content: Disrupting automated video-based llm annotations},
  author={Liu, Haitong and Gao, Kuofeng and Bai, Yang and Li, Jinmin and Shan, Jinxiao and Dai, Tao and Xia, Shu-Tao},
  booktitle={Proceedings of the IEEE/CVF Conference on Computer Vision and Pattern Recognition},
  pages={24056--24065},
  year={2025}
}

@inproceedings{shao2024lmdrive,
  title={Lmdrive: Closed-loop end-to-end driving with large language models},
  author={Shao, Hao and Hu, Yuxuan and Wang, Letian and Song, Guanglu and Waslander, Steven L and Liu, Yu and Li, Hongsheng},
  booktitle={Proceedings of the IEEE/CVF conference on computer vision and pattern recognition},
  pages={15120--15130},
  year={2024}
}

@inproceedings{Tao2024DyCokeD,
  title={Dycoke: Dynamic compression of tokens for fast video large language models},
  author={Tao, Keda and Qin, Can and You, Haoxuan and Sui, Yang and Wang, Huan},
  booktitle={Proceedings of the Computer Vision and Pattern Recognition Conference},
  pages={18992--19001},
  year={2025}
}

@InProceedings{Xing2024PyramidDropAY,
    author    = {Xing, Long and Huang, Qidong and Dong, Xiaoyi and Lu, Jiajie and Zhang, Pan and Zang, Yuhang and Cao, Yuhang and He, Conghui and Wang, Jiaqi and Wu, Feng and Lin, Dahua},
    title     = {Conical Visual Concentration for Efficient Large Vision-Language Models},
    booktitle = {Proceedings of the Computer Vision and Pattern Recognition Conference},
    month     = {June},
    year      = {2025},
    pages     = {14593-14603}
}

@inproceedings{Bolya2022TokenMY,
  title={Token Merging: Your {ViT} but Faster},
  author={Bolya, Daniel and Fu, Cheng-Yang and Dai, Xiaoliang and Zhang, Peizhao and Feichtenhofer, Christoph and Hoffman, Judy},
  booktitle={International Conference on Learning Representations},
  year={2023}
}

@inproceedings{Zhang2024SparseVLMVT,
  title={SparseVLM: Visual Token Sparsification for Efficient Vision-Language Model Inference},
  author={Zhang, Yuan and Fan, Chun-Kai and Ma, Junpeng and Zheng, Wenzhao and Huang, Tao and Cheng, Kuan and Gudovskiy, Denis and Okuno, Tomoyuki and Nakata, Yohei and Keutzer, Kurt and others},
  booktitle={International Conference on Machine Learning},
  year={2025}
}

@inproceedings{Shen2025FastVIDDD,
    title={Fast{VID}: Dynamic Density Pruning for Fast Video Large Language Models},
    author={Leqi Shen and Guoqiang Gong and Tao He and Yifeng Zhang and pengzhang liu and Sicheng Zhao and Guiguang Ding},
    booktitle={The Thirty-ninth Annual Conference on Neural Information Processing Systems},
    year={2025},
    url={https://openreview.net/forum?id=2xS4VtpApy}
}

@inproceedings{Huang2024PruneVidVT,
  title={Prunevid: Visual token pruning for efficient video large language models},
  author={Huang, Xiaohu and Zhou, Hao and Han, Kai},
  booktitle={Findings of the Association for Computational Linguistics: ACL 2025},
  pages={19959--19973},
  year={2025}
}

@inproceedings{Dao2023FlashAttention2FA,
  title={Flash{A}ttention-2: Faster Attention with Better Parallelism and Work Partitioning},
  author={Dao, Tri},
  booktitle={International Conference on Learning Representations},
  year={2024}
}

@inproceedings{Xiao2023EfficientSL,
  title={Efficient streaming language models with attention sinks},
  author={Xiao, Guangxuan and Tian, Yuandong and Chen, Beidi and Han, Song and Lewis, Mike},
  booktitle={The Twelfth International Conference on Learning Representations},
  year={2024}
}

@article{Vaswani2017AttentionIA,
  title={Attention is all you need},
  author={Vaswani, Ashish and Shazeer, Noam and Parmar, Niki and Uszkoreit, Jakob and Jones, Llion and Gomez, Aidan N and Kaiser, {\L}ukasz and Polosukhin, Illia},
  journal={Advances in neural information processing systems},
  volume={30},
  year={2017}
}

@inproceedings{Zhang2025SparsetoDenseAF,
  title={Sparse-to-Dense: A Free Lunch for Lossless Acceleration of Video Understanding in LLMs},
  author={Zhang, Xuan and Du, Cunxiao and Yu, Sicheng and Wu, Jiawei and Zhang, Fengzhuo and Gao, Wei and Liu, Qian},
  booktitle={Proceedings of the 63rd Annual Meeting of the Association for Computational Linguistics (Volume 2: Short Papers)},
  pages={734--742},
  year={2025}
}

@inproceedings{Ji2025SpecVLMES,
  title={Specvlm: Enhancing speculative decoding of video llms via verifier-guided token pruning},
  author={Ji, Yicheng and Zhang, Jun and Xia, Heming and Chen, Jinpeng and Shou, Lidan and Chen, Gang and Li, Huan},
  booktitle={Proceedings of the 2025 Conference on Empirical Methods in Natural Language Processing},
  pages={7216--7230},
  year={2025}
}

@inproceedings{FastVChen2024AnII,
  title={An image is worth 1/2 tokens after layer 2: Plug-and-play inference acceleration for large vision-language models},
  author={Chen, Liang and Zhao, Haozhe and Liu, Tianyu and Bai, Shuai and Lin, Junyang and Zhou, Chang and Chang, Baobao},
  booktitle={European Conference on Computer Vision},
  pages={19--35},
  year={2024},
  organization={Springer}
}

@inproceedings{Alvar2025DivPruneDV,
  title={Divprune: Diversity-based visual token pruning for large multimodal models},
  author={Alvar, Saeed Ranjbar and Singh, Gursimran and Akbari, Mohammad and Zhang, Yong},
  booktitle={Proceedings of the Computer Vision and Pattern Recognition Conference},
  pages={9392--9401},
  year={2025}
}

@inproceedings{Shang2024LLaVAPruMergeAT,
  title={Llava-prumerge: Adaptive token reduction for efficient large multimodal models},
  author={Shang, Yuzhang and Cai, Mu and Xu, Bingxin and Lee, Yong Jae and Yan, Yan},
  booktitle={Proceedings of the IEEE/CVF International Conference on Computer Vision},
  pages={22857--22867},
  year={2025}
}

@inproceedings{Leviathan2022FastIF,
  title={Fast inference from transformers via speculative decoding},
  author={Leviathan, Yaniv and Kalman, Matan and Matias, Yossi},
  booktitle={International Conference on Machine Learning},
  pages={19274--19286},
  year={2023},
  organization={PMLR}
}

@InProceedings{Chen2023AcceleratingLL,
  title = 	 {Fast Inference from Transformers via Speculative Decoding},
  author =       {Leviathan, Yaniv and Kalman, Matan and Matias, Yossi},
  booktitle = 	 {Proceedings of the 40th International Conference on Machine Learning},
  pages = 	 {19274--19286},
  year = 	 {2023},
  editor = 	 {Krause, Andreas and Brunskill, Emma and Cho, Kyunghyun and Engelhardt, Barbara and Sabato, Sivan and Scarlett, Jonathan},
  volume = 	 {202},
  series = 	 {Proceedings of Machine Learning Research},
  month = 	 {23--29 Jul},
  publisher =    {PMLR},
  url = 	 {https://proceedings.mlr.press/v202/leviathan23a.html},
}

@inproceedings{Xia2022SpeculativeDE,
  title={Speculative decoding: Exploiting speculative execution for accelerating seq2seq generation},
  author={Xia, Heming and Ge, Tao and Wang, Peiyi and Chen, Si-Qing and Wei, Furu and Sui, Zhifang},
  booktitle={Findings of the Association for Computational Linguistics: EMNLP 2023},
  pages={3909--3925},
  year={2023}
}

@inproceedings{Li2024EAGLESS, 
	author = {Yuhui Li and Fangyun Wei and Chao Zhang and Hongyang Zhang}, 
	title = {{EAGLE}: Speculative Sampling Requires Rethinking Feature Uncertainty}, 
	booktitle = {International Conference on Machine Learning},
	year = {2024}
}

@inproceedings{Li2024EAGLE2FI, 
	author = {Yuhui Li and Fangyun Wei and Chao Zhang and Hongyang Zhang}, 
	title = {{EAGLE-2}: Faster Inference of Language Models with Dynamic Draft Trees}, 
	booktitle = {Empirical Methods in Natural Language Processing},
	year = {2024}
}

@inproceedings{Li2025EAGLE3SU,
    author = {Yuhui Li and Fangyun Wei and Chao Zhang and Hongyang Zhang},
    title = {{EAGLE-3}: Scaling up Inference Acceleration of Large Language Models via Training-Time Test}, 
    booktitle = {Annual Conference on Neural Information Processing Systems},
    year = {2025}
}

@inproceedings{Liu2024ParallelSD,
    title={{PEARL}: Parallel Speculative Decoding with Adaptive Draft Length},
    author={Tianyu Liu and Yun Li and Qitan Lv and Kai Liu and Jianchen Zhu and Winston Hu and Xiao Sun},
    booktitle={The Thirteenth International Conference on Learning Representations},
    year={2025},
    url={https://openreview.net/forum?id=QOXrVMiHGK}
}

@inproceedings{Shen2025SpeculativeDV,
  title={SpecBranch: Speculative Decoding via Hybrid Drafting and Rollback-Aware Branch Parallelism},
  author={Shen, Yuhao and Shen, Junyi and Kong, Quan and Liu, Tianyu and Lu, Yao and Wang, Cong},
  booktitle={The Fourteenth International Conference on Learning Representations},
  year={2026}
}

@inproceedings{Gagrani2024OnSD,
  title={On speculative decoding for multimodal large language models},
  author={Gagrani, Mukul and Goel, Raghavv and Jeon, Wonseok and Park, Junyoung and Lee, Mingu and Lott, Christopher},
  booktitle={Proceedings of the IEEE/CVF Conference on Computer Vision and Pattern Recognition},
  pages={8285--8289},
  year={2024}
}

@article{Huang2025SpecVLMFS,
  title={SpecVLM: Fast Speculative Decoding in Vision-Language Models},
  author={Huang, Haiduo and Yang, Fuwei and Liu, Zhenhua and Yin, Xuanwu and Li, Dong and Ren, Pengju and Barsoum, Emad},
  journal={arXiv preprint arXiv:2509.11815},
  year={2025}
}

@inproceedings{Shao2025GrowingAT,
  title={Growing a twig to accelerate large vision-language models},
  author={Shao, Zhenwei and Wang, Mingyang and Yu, Zhou and Pan, Wenwen and Yang, Yan and Wei, Tao and Zhang, Hongyuan and Mao, Ning and Chen, Wei and Yu, Jun},
  booktitle={Proceedings of the IEEE/CVF International Conference on Computer Vision},
  pages={20064--20074},
  year={2025}
}

@inproceedings{Endo2024FeatherTT,
  title={Feather the throttle: Revisiting visual token pruning for vision-language model acceleration},
  author={Endo, Mark and Wang, Xiaohan and Yeung-Levy, Serena},
  booktitle={Proceedings of the IEEE/CVF International Conference on Computer Vision},
  pages={22826--22835},
  year={2025}
}

@inproceedings{Li2023MVBenchAC,
  title={Mvbench: A comprehensive multi-modal video understanding benchmark},
  author={Li, Kunchang and Wang, Yali and He, Yinan and Li, Yizhuo and Wang, Yi and Liu, Yi and Wang, Zun and Xu, Jilan and Chen, Guo and Luo, Ping and others},
  booktitle={Proceedings of the IEEE/CVF Conference on Computer Vision and Pattern Recognition},
  pages={22195--22206},
  year={2024}
}

@misc{VideoDetailCaption,
  author = {LMMs-Lab},
  title = {Video Detail Caption},
  year = {2024},
 note = {Accessed: 2024-11}
}

@InProceedings{Zhou2024MLVUAC,
    author    = {Zhou, Junjie and Shu, Yan and Zhao, Bo and Wu, Boya and Liang, Zhengyang and Xiao, Shitao and Qin, Minghao and Yang, Xi and Xiong, Yongping and Zhang, Bo and Huang, Tiejun and Liu, Zheng},
    title     = {MLVU: Benchmarking Multi-task Long Video Understanding},
    booktitle = {Proceedings of the Computer Vision and Pattern Recognition Conference},
    month     = {June},
    year      = {2025},
    pages     = {13691-13701}
}

@article{Wu2024LongVideoBenchAB,
  title={Longvideobench: A benchmark for long-context interleaved video-language understanding},
  author={Wu, Haoning and Li, Dongxu and Chen, Bei and Li, Junnan},
  journal={Advances in Neural Information Processing Systems},
  volume={37},
  pages={28828--28857},
  year={2024}
}

@inproceedings{Fu2024VideoMMETF,
  title={Video-mme: The first-ever comprehensive evaluation benchmark of multi-modal llms in video analysis},
  author={Fu, Chaoyou and Dai, Yuhan and Luo, Yongdong and Li, Lei and Ren, Shuhuai and Zhang, Renrui and Wang, Zihan and Zhou, Chenyu and Shen, Yunhang and Zhang, Mengdan and others},
  booktitle={Proceedings of the IEEE/CVF conference on computer vision and pattern recognition},
  pages={24108--24118},
  year={2025}
}

@InProceedings{cai2024medusa,
  title = 	 {Medusa: Simple {LLM} Inference Acceleration Framework with Multiple Decoding Heads},
  author =       {Cai, Tianle and Li, Yuhong and Geng, Zhengyang and Peng, Hongwu and Lee, Jason D. and Chen, Deming and Dao, Tri},
  booktitle = 	 {Proceedings of the 41st International Conference on Machine Learning},
  pages = 	 {5209--5235},
  year = 	 {2024},
  editor = 	 {Salakhutdinov, Ruslan and Kolter, Zico and Heller, Katherine and Weller, Adrian and Oliver, Nuria and Scarlett, Jonathan and Berkenkamp, Felix},
  volume = 	 {235},
  series = 	 {Proceedings of Machine Learning Research},
  month = 	 {21--27 Jul},
  publisher =    {PMLR},
  url = 	 {https://proceedings.mlr.press/v235/cai24b.html},

}

@inproceedings{du2024glide,
author = {Du, Cunxiao and Jiang, Jing and Yuanchen, Xu and Wu, Jiawei and Yu, Sicheng and Li, Yongqi and Li, Shenggui and Xu, Kai and Nie, Liqiang and Tu, Zhaopeng and You, Yang},
title = {GLIDE with a CAPE: a low-hassle method to accelerate speculative decoding},
year = {2024},
publisher = {JMLR.org},
booktitle = {Proceedings of the 41st International Conference on Machine Learning},
articleno = {465},
numpages = {17},
location = {Vienna, Austria},
series = {ICML'24}
}

@InProceedings{fu2024break,
  title = 	 {Break the Sequential Dependency of {LLM} Inference Using Lookahead Decoding},
  author =       {Fu, Yichao and Bailis, Peter and Stoica, Ion and Zhang, Hao},
  booktitle = 	 {Proceedings of the 41st International Conference on Machine Learning},
  pages = 	 {14060--14079},
  year = 	 {2024},
  editor = 	 {Salakhutdinov, Ruslan and Kolter, Zico and Heller, Katherine and Weller, Adrian and Oliver, Nuria and Scarlett, Jonathan and Berkenkamp, Felix},
  volume = 	 {235},
  series = 	 {Proceedings of Machine Learning Research},
  month = 	 {21--27 Jul},
  publisher =    {PMLR},
  url = 	 {https://proceedings.mlr.press/v235/fu24a.html},
}

@inproceedings{Zhao2024OuroborosGL,
  title={Ouroboros: Generating longer drafts phrase by phrase for faster speculative decoding},
  author={Zhao, Weilin and Huang, Yuxiang and Han, Xu and Xu, Wang and Xiao, Chaojun and Zhang, Xinrong and Fang, Yewei and Zhang, Kaihuo and Liu, Zhiyuan and Sun, Maosong},
  booktitle={Proceedings of the 2024 Conference on Empirical Methods in Natural Language Processing},
  pages={13378--13393},
  year={2024}
}

@inproceedings{wen2025stop,
  title={Stop Looking for “Important Tokens” in Multimodal Language Models: Duplication Matters More},
  author={Wen, Zichen and Gao, Yifeng and Wang, Shaobo and Zhang, Junyuan and Zhang, Qintong and Li, Weijia and He, Conghui and Zhang, Linfeng},
  booktitle={Proceedings of the 2025 Conference on Empirical Methods in Natural Language Processing},
  pages={9972--9991},
  year={2025}
}

@inproceedings{liu2025video,
  title={Video compression commander: Plug-and-play inference acceleration for video large language models},
  author={Liu, Xuyang and Wang, Yiyu and Ma, Junpeng and Zhang, Linfeng},
  booktitle={Proceedings of the 2025 Conference on Empirical Methods in Natural Language Processing},
  pages={1910--1924},
  year={2025}
}

@inproceedings{Tang2025AdaptiveKS,
  title={Adaptive keyframe sampling for long video understanding},
  author={Tang, Xi and Qiu, Jihao and Xie, Lingxi and Tian, Yunjie and Jiao, Jianbin and Ye, Qixiang},
  booktitle={Proceedings of the Computer Vision and Pattern Recognition Conference},
  pages={29118--29128},
  year={2025}
}

@article{shen2026double,
  title={Double: Breaking the Acceleration Limit via Double Retrieval Speculative Parallelism},
  author={Shen, Yuhao and Liu, Tianyu and Shen, Junyi and Wu, Jinyang and Kong, Quan and Huan, Li and Wang, Cong},
  journal={arXiv preprint arXiv:2601.05524},
  year={2026}
}

@inproceedings{yang2025visionzip,
  title={Visionzip: Longer is better but not necessary in vision language models},
  author={Yang, Senqiao and Chen, Yukang and Tian, Zhuotao and Wang, Chengyao and Li, Jingyao and Yu, Bei and Jia, Jiaya},
  booktitle={Proceedings of the IEEE/CVF Conference on Computer Vision and Pattern Recognition},
  pages={19792--19802},
  year={2025}
}

@article{Khaki2025SparseVILADV,
  title={SparseVILA: Decoupling Visual Sparsity for Efficient VLM Inference},
  author={Samir Khaki and Jun-Xian Guo and Jiaming Tang and Shang Yang and Yu-Kang Chen and Konstantinos N. Plataniotis and Yao Lu and Song Han and Zhi-Jian Liu},
  journal={2025 IEEE/CVF International Conference on Computer Vision (ICCV)},
  year={2025},
  pages={23784-23794},
  url={https://api.semanticscholar.org/CorpusID:282210595}
}

@misc{ji2026foresttreeslooselyspeculative,
      title={See the Forest for the Trees: Loosely Speculative Decoding via Visual-Semantic Guidance for Efficient Inference of Video LLMs}, 
      author={Yicheng Ji and Jun Zhang and Jinpeng Chen and Cong Wang and Lidan Shou and Gang Chen and Huan Li},
      year={2026},
      eprint={2604.05650},
      archivePrefix={arXiv},
      primaryClass={cs.CL},
      url={https://arxiv.org/abs/2604.05650}, 
}

@misc{zeng2026hybridkvhybridkvcache,
      title={HybridKV: Hybrid KV Cache Compression for Efficient Multimodal Large Language Model Inference}, 
      author={Bowen Zeng and Feiyang Ren and Jun Zhang and Xiaoling Gu and Ke Chen and Lidan Shou and Huan Li},
      year={2026},
      eprint={2604.05887},
      archivePrefix={arXiv},
      primaryClass={cs.AI},
      url={https://arxiv.org/abs/2604.05887}, 
}

@misc{zhang2026efficientinferencelargevisionlanguage,
      title={Efficient Inference for Large Vision-Language Models: Bottlenecks, Techniques, and Prospects}, 
      author={Jun Zhang and Yicheng Ji and Feiyang Ren and Yihang Li and Bowen Zeng and Zonghao Chen and Ke Chen and Lidan Shou and Gang Chen and Huan Li},
      year={2026},
      eprint={2604.05546},
      archivePrefix={arXiv},
      primaryClass={cs.CL},
      url={https://arxiv.org/abs/2604.05546}, 
}

@InProceedings{10.1007/978-3-032-37271-0_22,
author="Kong, Quan
and Xiao, Yanru
and Shen, Yuhao
and Wang, Cong",
editor="Favaro, Paolo
and Kukelova, Zuzana
and Maki, Atsuto
and Rohrbach, Anna
and Schindler, Konrad
and Tombari, Federico",
title="Vision-TTT: Efficient and Expressive Visual Representation Learning with Test-Time Training",
booktitle="Computer Vision -- ECCV 2026",
year="2026",
publisher="Springer Nature Switzerland",
address="Cham",
pages="401--420",
isbn="978-3-032-37271-0"
}
